%% file: main.tex
\documentclass[10pt]{article}
\usepackage[margin=1in]{geometry}
\usepackage{enumitem}
\setlist[itemize]{noitemsep, topsep=-5pt}
\setlist[enumerate]{noitemsep, topsep=-5pt}

\PassOptionsToPackage{numbers, compress}{natbib}

\usepackage[utf8]{inputenc} 
\usepackage[T1]{fontenc}    
\usepackage{hyperref}       
\usepackage{url}            
\usepackage{booktabs}       
\usepackage{amsfonts}       
\usepackage{nicefrac}       
\usepackage{microtype}      
\usepackage{xcolor}         
\usepackage{colortbl}
\usepackage{mathtools}
\usepackage{algorithm}
\usepackage{algpseudocode}
\usepackage{tikz}
\usepackage{caption}
\usepackage{subcaption} 
\usepackage{amsfonts, amsmath, amsthm}       
\usepackage{amssymb}  
\usepackage{multirow} 
\usepackage{booktabs} 
\usepackage{accents}
\usepackage{comment}
\usepackage{wrapfig}
\usepackage{enumitem}
\usepackage{array}
\usepackage{tabularx}
\usepackage{wrapfig}

\usepackage{tabularx}
\newcolumntype{Y}{>{\centering\arraybackslash}X}

\usepackage{natbib} 
\newtheorem{lemma}{Lemma}
\newtheorem{definition}{Definition}
\newtheorem{theorem}{Theorem}

\newtheorem{remark}{Remark}

\newtheorem*{theorem*}{Theorem}
\newtheorem*{proposition*}{Proposition}

\newcommand{\matrixpropag}{{S}}

\newcommand{\sintra}{S^{\mathrm{intra}}}
\newcommand{\sextra}{S^{\mathrm{inter}}}
\newcommand{\sextrac}{S^{\mathrm{inter}}_c}

\newcommand{\newmatrixpropag}{{S_c}}
\newcommand{\preservedspace}{\mathcal{R}}

\title{CoRe-GNN: Multilevel Message passing \\ on Coarsened graphs}

\author{%
 Antonin Joly \\
  CNRS, IRISA, Rennes, FRANCE\\
  \texttt{antonin.joly@irisa.fr} \\
  \and
  Nicolas Keriven\\
  CNRS, IRISA, Rennes, FRANCE\\
  \texttt{nicolas.keriven@cnrs.fr} \\
  \and
  Aline Roumy \\
  INRIA, Rennes, FRANCE\\
  \texttt{aline.roumy@inria.fr}
}

\begin{document}

\maketitle

\begin{abstract}
Training Graph Neural Networks on large graphs is challenged by the memory cost of storing all node representations across layers. We show that several existing scalable approaches can be written as structured modifications of the GNN propagation matrix, providing a unified perspective that exposes their respective limitations.
In particular, graph coarsening replaces it by a low-rank approximation that enables spectral guarantees
but assigns uniform representations to clustered nodes, while Cluster-GCN restricts the propagation matrix to intra-cluster connections that allow efficient batching but sever long-range information. These are complementary failures of the \emph{same} decomposition of the graph into groups of nodes. To obtain the best of both worlds, we propose \textbf{CoRe-GNN}, which performs both propagations in parallel at each layer: a coarsened inter-cluster term capturing long-range structure, and a local intra-cluster term preserving per-node discriminability. We prove that CoRe-GNN inherits analogous approximation guarantees to those of graph coarsening, and introduce a natural cluster-based \emph{batching scheme} that scales to graphs with millions of nodes. 
On node classification benchmarks spanning homophilic, heterophilic, large-scale, and long-range graphs, CoRe-GNN outperforms both graph coarsening and Cluster-GCN baselines. 
Notably, CoRe-GNN reaches competitive accuracy on \emph{long-range} tasks, while remaining memory-efficient through batching.
\end{abstract}

\section{Introduction}
Graph Neural Networks (GNNs) \cite{scarselli2008graph,kipf2016GCN,Bronstein2021} have established themselves as the dominant paradigm for learning on graph-structured data. 
GNNs are deep architectures on graphs that rely on the \textbf{Message-Passing} paradigm \cite{Gilmer2017}: at each layer, the representation $H^{(l)}_i \in \mathbb{R}^{d_l}$ of each node $1\leq i \leq N$, is updated by \emph{aggregating} and \emph{transforming} the representations of its neighbors at the previous layer $\{H^{(l-1)}_j\}_{j \in \mathcal{N}(i)}$, where $\mathcal{N}(i)$ is the neighborhood of $i$. In most examples, this aggregation can be represented as a \emph{multiplication} of the node representation matrix $H^{(l-1)} \in \mathbb{R}^{N \times d_{l-1}}$ by a \emph{propagation matrix} $\matrixpropag \in \mathbb{R}^{N \times N}$ related to the graph structure, followed by a learnable linear transformation.
This aggregation can be expressed compactly as: starting from initial node features $H^{(0)} \in \mathbb{R}^{N \times d_0}$, a GNN with $K$ layers computes

\begin{equation}\label{eq:gnn}
   H^{(l)} = \sigma\left( \matrixpropag H^{(l-1)} \theta^{(l)} \right), \qquad l = 1,
   \ldots, K,
\end{equation}
where $\sigma$ is a pointwise activation function and $\theta^{(l)} \in \mathbb{R}^{d_{l-1} \times d_l}$ are learned parameters. Classical choices include $S = D^{-1}A$ (mean aggregation) or the normalized adjacency $S = (D+I)^{-1/2}(A+I)(D+I)^{-1/2}$, which corresponds to the classical GCNconv layer \cite{kipf2016GCN}. Note that we do not include skip connections here for later ease of exposition. For the rest of the paper, unless otherwise precised, we set  
$S = (D+I)^{-1/2}(A+I)(D+I)^{-1/2}$ as in~\cite{kipf2016GCN}.

Training a GNN~\eqref{eq:gnn} on a graph of $N$ nodes requires storing the intermediate representation of each of the $K$ layers, a memory cost that grows linearly with the number of nodes $N$, saturating standard GPU memory for large-scale graphs~\cite{hu2020open, hamilton2017inductive}. 
This memory bottleneck affects not only training but also inference of the network on such graphs.
Unlike in most machine learning settings, naive node batching does not decompose the computation in~\eqref{eq:gnn}: the representation $H^{(l)}_i$ of node $i$ at layer $l$ depends on the features of its entire $l$-hop neighborhood. 
For graphs with low average path length, this neighborhood grows rapidly with depth. As emphasized by~\cite{zeng2021decoupling}, in OGBN-products \cite{hu2020open} the 4-hop neighborhood of one node is in average  $\approx 0.6M$ nodes ($25\%$ of the full graph), leading to memory issues on GPU. This scalability bottleneck has motivated two main families of approaches: \emph{sampling-based} methods, which truncate the $l$-hop neighborhood to a fixed budget per layer, and \emph{compressed-graph} methods, which train the GNN on a structurally reduced version of the original graph. In this paper we focus on the latter.

\paragraph{Compressed-graph methods and their limitations.}
Compressed-graph methods trains the GNN on a structurally reduced version of the original graph.
Among them, \emph{graph coarsening}~\cite{loukas2019graph, kumar2023featured, kataria2024ugc,dickens2024graph, cohens2025pectral,joly2024graph, joly2025taxonomy} maps the original nodes to so-called \emph{supernodes}, forming a smaller graph that retains some properties of the original graph. 
 While, all these methods \emph{train} the GNN on the coarsened graph, they may differ in how they perform inference at test time. Some~\cite{kumar2023featured, kataria2024ugc, cohens2025pectral, dickens2024graph} apply the learned weights directly to the original graph at inference time. However: a) they generally lack guarantees that these weights transfer well to the original structure, and b) some graphs might even be too large for inference. Other works~\cite{joly2024graph, joly2025taxonomy} evaluate entirely within the coarsened space, ensuring consistency between training and inference, and this is the method that we adopt here. 
 However, a fundamental limitation of this approach is that all nodes mapped to the same supernode share the same prediction, which might be limiting especially in heterophilic situations.

In parallel, \emph{sparsification approaches} restrict $\matrixpropag$ to a sparser version computed once as preprocessing. \emph{Cluster-GCN}~\cite{chiang2019cluster} takes the converse approach to graph coarsening: rather than collapsing nodes across clusters, it restricts the GNN to intra-cluster connections only, replacing $\matrixpropag$ by a block-diagonal restriction. In addition to saving (some) memory, the main advantage of this approach is to easily allow for \emph{batching} during training, which is normally known to be rather problematic for GNNs: $k$-hop subgraphs are required around each batched nodes, which can easily be the entire graph. By forming disconnected subgraphs that can be independently loaded, Cluster-GCN solves this problem. However, the main limitation of such an approach is that inter-cluster edges are discarded, severing long-range dependencies. This might be limiting on long-range tasks.

\paragraph{Contributions.}
We propose \textbf{CoRe-GNN} (\textbf{Co}arsen and \textbf{Re}store), an architecture that performs in parallel an inter-cluster propagation on the coarsened graph and an intra-cluster propagation on the original node space, coupling the two across layers. The architecture is agnostic to the choice of the coarsening algorithm, which is treated as a tunable hyperparameter. This design addresses simultaneously the accuracy ceiling of graph coarsening and the long-range blindness of Cluster-GCN. First, we provide a unified matrix formulation of compressed-graph GNN training, showing that Cluster-GCN and graph coarsening correspond to complementary structured modifications of the propagation matrix $\matrixpropag$, clarifying their limitations and motivating CoRe-GNN. 
Second, we introduce a natural batching scheme that scales CoRe-GNN to graphs with millions of nodes. Third, we prove that CoRe-GNN inherits guarantees over the preservation of smooth signal propagation, with the same structure as those of graph coarsening but with different multiplicative constants that we empirically evaluate. Fourth, we evaluate CoRe-GNN on homophilic, heterophilic, large-scale, and long-range node classification benchmarks, consistently outperforming Cluster-GCN and graph coarsening baselines. An illustrative diagram of the method can be found in Fig.~\ref{fig:core_gnn_layer}.

\input{tikz_figure/coregnn_layer_flow.tikz}

\paragraph{Related work.}
Of course, there are many other classes of methods to handle large graphs. \emph{Sampling-based} methods address the $l$-hop explosion by approximating the neighborhood of a node at each layer. GraphSAGE~\cite{hamilton2017inductive} samples a fixed number of neighbors per layer for each node in a batch. The number of sampled neighbors grows exponentially with the number of layers making deep architectures expensive, and a large body of work has been devoted to mitigating this cost \cite{chen2018fastgcn, zou2019layer, gasteiger2022influence}.
An alternative is to sample entire subgraphs rather than individual neighbors \cite{zeng2019graphsaint, zeng2021decoupling}.
In general, coarsening-based and sampling-based methods differ structurally: while coarsening is performed once before training as a \emph{pre-processing}, sampling is generally performed at every step of training, thus requiring to retain access to the original graph.
\emph{Graph pooling} methods~\cite{ying2018hierarchical, bianchi2020spectral} are designed for graph-level classification and do not support node-level prediction. Most are differentiable and learned end-to-end, in contrast to the fixed pre-computed partitions used in coarsening-based approaches. Graph U-Net~\cite{gao2019graph} extends pooling to node classification but its differentiable coarsening is memory-intensive and not designed for scalability.
\emph{Graph condensation}~\cite{jin2022condensing} creates a small synthetic graph by aligning GNN gradients between the compressed and original graphs. While effective at small reduction ratios, it does not preserve the original topology and evaluation on the full graph at inference remains memory-intensive. 
\emph{Graph summarization} encompasses broader compression techniques: VQ-GNN~\cite{ding2021vq} compresses node representations via vector quantization, while methods based on overlapping cluster structures~\cite{finkelshtein2024learning, kouchly2025efficient} allow nodes to belong to multiple communities simultaneously. 
Several works use multi-level aggregation across different graph resolutions~\cite{zhong2023hierarchical, li2025partition}, but rely on complex aggregation schemes that are not batchable and do not target scalability to large graphs. In our work, the two levels of granularity serve a different purpose: resolving the accuracy ceiling of coarsening-only methods while preserving the scalability benefits of the cluster partition. 
Close to our work, FiT-GNN~\cite{roy2024fit} and~\cite{vonessen2024next} treat coarsened nodes as auxiliary nodes appended to the original graph, keeping all computations in the original node space. This avoids the dimension change between original and coarsened spaces but does not benefit from the computational advantages of training and inferring within the coarsened space.

\section{Compressed-Graph GNNs as Propagation Matrix Modifications}

Several compressed-graph GNN methods can be written as modifications of the standard update~\eqref{eq:gnn}, replacing $\matrixpropag$ by a structurally modified matrix to reduce memory and computation:
\begin{equation}\label{eq:general}
H^{(l)} = \sigma \left( S_{\mathrm{compressed}}\, H^{(l-1)}\theta^{(l)} \right)
\end{equation}
We detail below two such methods, Coarsen-GNN~\cite{joly2024graph, joly2025taxonomy} and Cluster-GCN~\cite{chiang2019cluster}, which operate on complementary modifications of the same partition-based decomposition of $\matrixpropag$. Note that both Coarsen-GNN and Cluster-GCN refer here to training and evaluating a GNN on, respectively, a coarsened or a clustered version of a graph, computed as a pre-processing with a dedicated algorithm (see Sec.~\ref{sec:expe}), \emph{not} to using a GNN to \emph{learn} the coarsening or the clustering.

\subsection{Coarsen-GNN}

Graph coarsening~\cite{loukas2019graph, joly2025taxonomy} reduces a graph $G$ with $N$ nodes to a coarsened graph $G_c$ with $n < N$ nodes. The proportion of reduction achieved is measured by the \textbf{coarsening ratio} $r = 1- \frac{n}{N}$. The mapping from $G$ to $G_c$ is obtained by grouping set of nodes in $G$ to form so-called \emph{supernodes} in $G_c$.  
Following~\cite{joly2025taxonomy}, we describe the coarsening through two matrices: a \textbf{reduction matrix} $P \in \mathbb{R}^{n \times N}$, which allows going from the original $N$-dimensional space to the coarsened $n$-dimensional space, and a \textbf{lifting matrix} $Q \in \mathbb{R}^{N \times n}$, which allows going back. Perhaps surprisingly, it is more convenient to define the coarsening through the lifting matrix $Q$, while $P$ will naturally follow \cite{joly2025taxonomy}. First, $Q$ must be \textbf{well-partitioned}: it must have exactly one non-zero entry per row, ensuring each original node $i$ belongs to exactly one supernode $\ell$, when $Q_{i\ell}>0$. Second, the non-zero entries of $Q$ are chosen such that, given some symmetric Laplacian $L$ on $G$, the coarsened Laplacian $L_c = Q^\top L Q$ is well-defined as a true Laplacian. Examples include the combinatorial Laplacian $D - A$ or the self-loop normalized Laplacian~\cite{joly2025taxonomy} $L = (D+I)^{-1/2}(D-A)(D+I)^{-1/2}$, which is related to the propagation matrix $S$ in GCN \cite{kipf2016GCN}. We generally adopt the latter in this paper. 
There are several possible choices for $P$ \cite{joly2025taxonomy}, but the most natural is to set $P = Q^+$, the Moore-Penrose pseudo-inverse of $Q$. Since $Q$ is well-partitioned, $Q^+ = (Q^\top Q)^{-1} Q^\top$ has only $N$ non-zero elements, resulting in linear-complexity lifting and reduction operations. In this paper we will always consider that $P=Q^+$.

Coarsen-GNN performs the entire propagation directly on the coarsened graph of size $n$. Considering node representations $H_c^{(l)} \in \mathbb{R}^{n \times d}$ on the coarsened graph, with initial $H^{(0)}_c = P H^{(0)}$, the propagation is naturally
\begin{equation}\label{eq:coarsenc}
\text{Coarsen-GNN} : \quad H_c^{(l)} = \sigma\!\left(S_c\, H_c^{(l-1)}\theta^{(l)}\right)
\end{equation}
where $S_c = PSQ \in \mathbb{R}^{n \times n}$ is the propagation matrix introduced in~\cite{joly2024graph}, followed by a final lifting by $Q$. Since $Q$ commutes with any pointwise activation $\sigma$ (Lemma~\ref{lem:commute}) and $PQ = I_n$, \eqref{eq:coarsenc} can equivalently be rewritten as a propagation on the original graph  with the matrix $\Pi S \Pi$, where $\Pi = QP \in \mathbb{R}^{N \times N}$ is the coarsen-lift operator: 
\begin{equation}\label{eq:coarsenorigin}
H^{(l)} = \sigma\!\left(\Pi S \Pi\, H^{(l-1)}\theta^{(l)}\right) = Q\sigma\!\left(\newmatrixpropag \,P H^{(l-1)}\theta^{(l)}\right)
\end{equation}
This shows that Coarsen-GNN fits the general framework~\eqref{eq:general} with $S_{\mathrm{compressed}} = \Pi S \Pi$. 
Of course, in practice, Coarsen-GNN propagates $H_c^{(l)}$ and only maps the final predictions back to the original nodes, but this view is a stepping stone towards our more general architecture. 
Moreover, it outlines the main limitation of Coarsen-GNN: since $\Pi$ is constant by block, all nodes within the same supernode receive the same prediction. As most coarsening algorithms merge neighboring nodes, this is particularly limiting on heterophilic graphs where adjacent nodes often carry different labels.

\subsection{Cluster-GCN}

Cluster-GCN~\cite{chiang2019cluster} also operates on a partitioning of the graph, here more classically referred to as clusters (instead of ``supernodes'').
Rather than projecting the node representations into a smaller space, it restricts the propagation to \emph{intra-cluster connections} only, that is, it removes all edges between clusters. Denoting $\sintra$ the restriction of $\matrixpropag$ to intra-cluster connections, Cluster-GCN replaces $\matrixpropag$ by $\sintra$ at every layer:  
\begin{equation}\label{eq:cluster}
\text{Cluster-GCN} : \quad H^{(l)} = \sigma\!\left(\sintra\, H^{(l-1)}\theta^{(l)}\right)
\end{equation}

In addition to saving memory by discarding some connections, the main advantage of Cluster-GCN is that, since $\sintra$ is block-diagonal with one block per cluster, \emph{batching} is straightforward: at each training step, we can select a subset $\mathcal{B}$ of clusters and loads only the corresponding blocks of $\sintra$ and node representations. 
However, Cluster-GCN severs all inter-cluster and long-range dependencies, which severely limits its performance on long-range tasks~\cite{liang2025towards}.

Figure~\ref{fig:prop_matrices_v2} illustrates the propagation matrices of the different methods.

\begin{remark}
Note that the original Cluster-GCN paper~\cite{chiang2019cluster} also considers \emph{merging} several randomly drawn clusters per batch during training, but similar to Coarsen-GNN here we consider that the partitioning is fixed, done as a preprocessing step. 
\end{remark}

\input{tikz_figure/prop_matrices_5cluster.tikz}

\section{A new architecture CoRe-GNN} 

 We have seen above that Coarsen-GNN assigns a single representation per supernode, limiting per-node discriminability, while Cluster-GCN discards all inter-cluster connections, severing long-range dependencies. In some sense, these are \emph{complementary} limitations, which we simultaneously address with the proposed \textbf{CoRe-GNN}.
It combines both strategies in parallel: it adds intra-cluster connections to Coarsen-GNN, or equivalently, extends Cluster-GCN with coarsened inter-cluster propagation. Denoting $\sextra = \matrixpropag - \sintra$ the inter-cluster part of $\matrixpropag$ and $\sextrac = P\sextra Q \in \mathbb{R}^{n \times n}$ its coarsened counterpart, CoRe-GNN performs both propagations in parallel at each layer: 
\begin{equation}\label{eq:core-gnn}
\text{CoRe-GNN} : \quad H^{(l)} = Q\,\sigma\!\left(\sextrac\, P H^{(l-1)}\theta_c^{(l)}\right) + \sigma\!\left(\sintra\, H^{(l-1)}\theta^{(l)}\right)
\end{equation}
The name CoRe-GNN reflects this two-step structure, where the inter-cluster term \textbf{co}arsens the signal to capture long-range structure, while the intra-cluster term \textbf{re}stores per-node discriminability lost by the coarsening. 
Unlike Coarsen-GNN which maps back to the original dimension only at the final layer, CoRe-GNN couples the two spaces at every layer via $Q$ and $P$, maintaining two parallel states $H_c^{(l)}$ and $H^{(l)}$, which directly motivates the batching scheme presented right after.
An illustrative diagram is provided in Figure~\ref{fig:core_gnn_layer}. The weight matrices $\theta^{(l)}$ and $\theta_c^{(l)}$ are independent: they respectively treat close-by intra-connections and long-range coarsened connections. Several design choices including weight sharing are studied in Appendix~\ref{app:ablation}.

\begin{remark}[CoRe-GNN strictly generalizes Cluster-GCN and approximates Coarsen-GNN]
\label{rem:generalize}
Setting $\theta_c^{(l)} = 0$ in CoRe-GNN recovers Cluster-GCN exactly. The converse direction is more subtle: setting $\theta^{(l)} = 0$ does not recover Coarsen-GNN, as $\sextrac = P\sextra Q$ differs from $\newmatrixpropag = PSQ$ by the term $P\sintra Q$. This term is block-diagonal as shown in Lemma~\ref{lem:block_diagonal}.
\end{remark}

\subsection{Natural Batching extension}

CoRe-GNN is designed to scale to large graphs where classical GCN runs out of GPU memory. Although the proposed architecture is already lighter than full GCN by replacing $\matrixpropag$ by the sparser combination $\sintra + Q\sextrac P$, this alone is not sufficient to handle graphs with millions of nodes. However, akin to Cluster-GCN, CoRe-GNN is amenable to a natural \emph{batching} strategy, by only selecting some clusters for intra-propagation. The batching scheme maintains two parallel states: $H_c^{(l)}$ for all super-nodes in the coarsened space, and $H_{\mathcal{B}}^{(l)}$ \emph{for the batched nodes only},  
in the original space. At each step, the inter-cluster term runs over all $n$ super-nodes cheaply since $n \ll N$, while only the clusters in the current batch $\mathcal{B}$ perform the full intra-cluster propagation. For non-batch super-nodes, the intra-cluster term is approximated by a self-loop proxy in the coarsened space, which is equivalent to lifting the super-node to the original space, applying a self-loop node by node, and reducing back.

\begin{algorithm}[h]
\small
\caption{CoRe-GNN: Batched Forward Pass}\label{alg:core-batched}
\begin{algorithmic}[1]
\Require $H^{(0)}_c =PH^{(0)} \in \mathbb{R}^{n \times d}$,\;
         $H^{(0)}_{\mathcal{B}} \in \mathbb{R}^{N_\mathcal{B} \times d}$,\;
         $P_\mathcal{B}, Q_\mathcal{B}$ \textit{(restricted to batch)},\;
         $\sextrac, \sintra[\mathcal{B}]$ \textit{(precomputed)},\;
         batch $\mathcal{B} \subseteq [n]$,\;
         weights $(\theta^{(l)}, \theta_c^{(l)})_{l=1}^K$
\For{$l = 1, \ldots, K-1$}
  \State $H_c^{(l)} \leftarrow \sigma\!\left(\sextrac\, H_c^{(l-1)}\theta_c^{(l)}\right)$
         \hfill\textit{// all $n$ super-nodes, coarsened space}
  \State $H_{\mathcal{B}}^{(l)} \leftarrow \sigma\!\left(\sintra[\mathcal{B}]\, H_{\mathcal{B}}^{(l-1)}\theta^{(l)}\right)$
         \hfill\textit{// intra-cluster, batch nodes only}
  \State $H_{\mathcal{B}}^{(l)} \leftarrow H_{\mathcal{B}}^{(l)} +  Q_\mathcal{B}\, H_c^{(l)}[\mathcal{B}]$
         \hfill\textit{// lift inter-cluster to batch nodes}
  \State $H_c^{(l)}[\mathcal{B}] \leftarrow P_\mathcal{B}\, H_{\mathcal{B}}^{(l)}$
         \hfill\textit{// reduce batch nodes back to super-nodes}
  \State $H_c^{(l)}[\bar{\mathcal{B}}] \leftarrow \sigma\!\left(H_c^{(l-1)}[\bar{\mathcal{B}}]\,\theta^{(l)}\right) +  H_c^{(l)}[\bar{\mathcal{B}}]$
         \hfill\textit{// self-loop proxy for non-batch super-nodes}
\EndFor
\State \textbf{return} $H_{\mathcal{B}}^{(K)}$, $H_c^{(K)}$
\end{algorithmic}
\end{algorithm}

Algorithm~\ref{alg:core-batched} summarizes the batched forward pass. A custom sampler groups clusters until a target node count is reached, ensuring constant-size batches despite variable cluster sizes. The intra-cluster blocks are pre-split at loading time and remain fixed during training, so no graph partitioning or subgraph sampling is performed during training. 

\paragraph{Exact inference.}
Without batches, standard full-graph \emph{inference} of CoRe-GNN may be out of memory on very large graphs such as Reddit and OGBN-Products~\cite{hu2020open}, even without retaining gradients.
However, we can derive a non-trivial exact inference procedure that avoids this bottleneck without any approximation, without resorting to the self-loop proxy used during training. For each layer, the inter-cluster term is computed once over all $n$ super-nodes on GPU via $\sextrac$, while the intra-cluster term is computed cluster by cluster, loading each block $\sintra[\mathcal{B}]$ from CPU to GPU on demand and discarding it immediately after. This strategy is not applicable during training, since backpropagation requires retaining the full computational graph across all layers, which causes out-of-memory crashes on large graphs. Detailed pseudo-codes for exact inference (Appendix~\ref{app:inference}), Coarsen-GNN training, CoRe-GNN full batch and batched training are provided in Appendix~\ref{app:pseudocode}.

\subsection{Complexity}
\label{sec:complexity}

CoRe-GNN replaces the full propagation matrix $\matrixpropag$ by the sparser combination $\sintra + Q\sextrac P$. Classical GCN requires $\mathcal{O}(KEd + KNd^2)$ operations per epoch, where $K$ is the number of layers, $E = \mathrm{nnz}(\matrixpropag)$ the number of edges, and $d$ the hidden dimension. CoRe-GNN substitutes $E$ by $E_{\mathrm{intra}} + E_c$, where $E_{\mathrm{intra}} = \mathrm{nnz}(\sintra) \leq E$ and $E_c = \mathrm{nnz}(\sextrac) \ll E$ since $\sextrac \in \mathbb{R}^{n \times n}$ with $n \ll N$. The lifting and reduction operations via $P$ and $Q$, which have only $N$ non-zero elements each, add a negligible $\mathcal{O}(N)$ cost.
Whether CoRe-GNN is cheaper than classical GCN depends on the graph structure and the coarsening ratio, as analyzed in Appendix~\ref{app:efficiency_regimes}. The batched variant provides a more systematic advantage: memory is reduced to $\mathcal{O}(K(n+b)d)$, independent of $N$, enabling training on graphs that would otherwise exceed GPU memory. A detailed complexity table of all methods is also provided in Appendix~\ref{app:complexity}.

\subsection{Propagation guarantees for CoRe-GNN}

In this section, we derive some theoretical guarantees that bound the error between message-passing on the original graph and message-passing in CoRe-GNN. For this, we adapt the results in \cite{joly2024graph} for Coarsen-GNN. Note that, to our knowledge, there is no theoretical analysis for Cluster-GCN that we could adapt for CoRe-GNN, 
and we leave such a complementary approach for future work.

In this section $L$ denotes a symmetric Laplacian of the original graph, the propagation matrix is $S = I_N - L$, and $Q$ is such that the coarsened Laplacian $L_c = Q^\top L Q$ is well-defined. Recall that $P = Q^+$ and $\Pi = QP$.
For a symmetric positive semi-definite matrix $L$, we define the Mahalanobis semi-norm $\|x\|_L = \sqrt{x^\top L x}$ and the associated pseudo-operator norm $\|M\|_L = \|L^{1/2} M L^{-1/2}\|_{op}$.

\paragraph{Reminder on Coarsening guarantee.}
To measure the quality of a coarsening, Loukas~\cite{loukas2019graph} introduced the \emph{Restricted Spectral Approximation} (RSA) constant, which measures the loss of information from a signal $x$ to its coarsened and re-lifted counterpart $\tilde{x} = \Pi x$. Since $\Pi$ is at most of rank $n < N$, only a subspace $\preservedspace$ of $\mathbb{R}^N$ can be preserved exactly.
\begin{definition}[Restricted Spectral Approximation constant]
Given a subspace $\preservedspace \subset \mathbb{R}^N$, a Laplacian $L$, a lifting matrix $Q$ and a reduction matrix $P$, the RSA constant is:
\begin{equation}\label{eq:rsa}
\epsilon_{L,Q,P,\preservedspace} = \sup_{x\in \preservedspace,\, \|x\|_L =1} \lVert x - QP x \rVert_{L}
\end{equation}
\end{definition}
Classically, $\preservedspace$ is spanned by the eigenvectors associated with the smallest eigenvalues of $L$, corresponding to the smoothest signals on the graph. Then, the authors in \cite{joly2024graph} show that, with our choice of propagation matrix $\newmatrixpropag = PSQ$, one step of message-passing in Coarsen-GNN \eqref{eq:coarsenc} is provably close to message-passing on the original graph,
in terms of the RSA constant $\epsilon_{L,Q,P,\preservedspace}$.

\begin{theorem}[Propagation on coarsened graph~\cite{joly2024graph}]
\label{thm:propag}
Under the assumption that $\Pi$ and $S$ are both $\ker(L)$-preserving and $S$ is $\preservedspace$-preserving, for all $x \in \preservedspace$:
\begin{equation}
\lVert Sx - Q\newmatrixpropag Px \rVert_{L} \leq \epsilon_{L,Q,P,\preservedspace} \lVert x \rVert_{L} \left( C_S + C_\Pi \right)
\end{equation}
where $C_S := \lVert S \rVert_{L}$, $C_\Pi := \lVert \Pi S \rVert_{L}$, and $M$ is said $E$-preserving if $Mx \in E$ for all $x \in E$.
\end{theorem}

The various assumptions of Theorem \ref{thm:propag} are discussed in the original paper~\cite{joly2024graph}. The RSA constant controls how well the coarsened graph preserves smooth signals, which directly justifies using coarsening algorithms that minimize the RSA~\cite{loukas2019graph} rather than other spectral criteria.

\paragraph{Guarantees for CoRe-GNN.}
Theorem~\ref{thm:propag} extends naturally to CoRe-GNN, 
under the same assumptions but with different constants. The full derivation is in Appendix~\ref{app:proof_coregnn}. 

\begin{theorem}[Propagation guarantee for CoRe-GNN]
\label{thm:propagcoregnn}
Under the assumption that $\Pi$ and $\sextra$ are both $\ker(L)$-preserving and $\sextra$ is $\preservedspace$-preserving, for all $x \in \preservedspace$:
\begin{equation}
\lVert Sx - \left( \sintra x + Q\sextrac Px \right) \rVert_{L} \leq \epsilon_{L,Q,P,\preservedspace} \lVert x \rVert_{L} \left( C_S' + C_\Pi' \right)
\end{equation}
where $C_S' := \lVert \sextra \rVert_{L}$, $C_\Pi' := \lVert \Pi\sextra \rVert_{L}$. 
\end{theorem}

The constants $C_S'$ and $C_\Pi'$ now depend on $\sextra$ rather than $S$. Comparing the two results, CoRe-GNN has better preservation guarantees than Coarsen-GNN when $C_S' + C_\Pi' < C_S + C_\Pi$. While there is no guarantee that this is always the case, as these constants involves operator norms, note that if we consider Frobenius norms instead we would easily have $\lVert \sextra \rVert_{\mathrm{Fro}} \leq \lVert S \rVert_{\mathrm{Fro}}$, hinting that the constants are indeed improved for CoRe-GNN.
A notable difference is that $C_S' = \lVert \sextra \rVert_L$ depends on the coarsening, while $C_S = \lVert S \rVert_L$ is fixed by the graph. As the coarsening ratio increases, $\sextra$ becomes sparser and $C_S'$ is expected to decrease.  

Figure~\ref{fig:bound_comparison_sbm} shows the constants $C_S$, $C_\Pi$, $C'_S$, $C'_\Pi$, the full theoretical upper bound, as well as the empirical propagation error averaged over 100 signals sampled from the low-frequency eigenvectors of $L$, as a function of the coarsening ratio $r$, on an SBM synthetic graph. We observe that the theoretical bound and the propagation error are consistently lower for CoRe-GNN across all coarsening ratios. On real graphs (see Appendix~\ref{app:constants}), the picture is more nuanced: the constants for CoRe-GNN are not always smaller, but the empirical propagation error generally remains lower, showing that the upper bounds of Theorem \ref{thm:propag} and \ref{thm:propagcoregnn} tend to be rather loose on real data.

\begin{figure}[t]
  \centering
  \begin{subfigure}[b]{0.24\linewidth}
    \includegraphics[height=2.8cm]{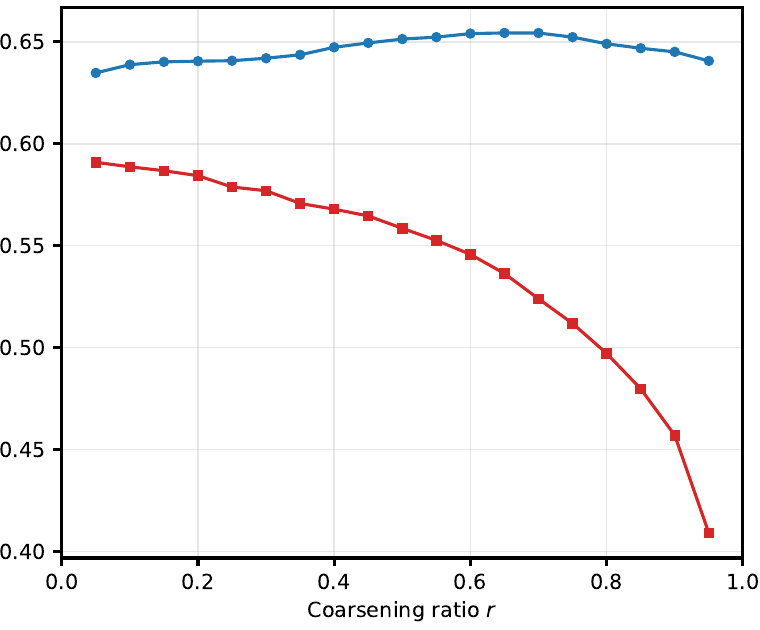}
    \caption*{$\mathcal{C}_\Pi$ terms}
  \end{subfigure}\hfill
  \begin{subfigure}[b]{0.24\linewidth}
    \includegraphics[height=2.8cm]{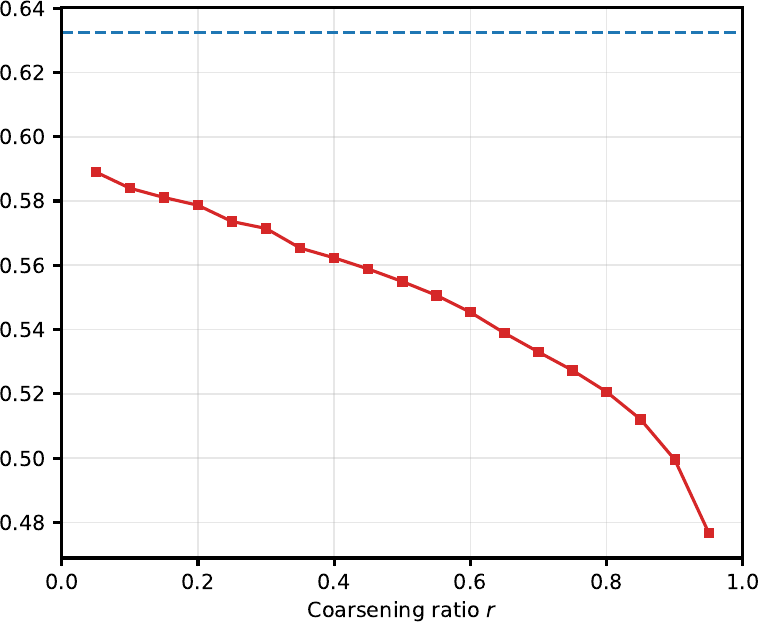}
    \caption*{$\mathcal{C}_S$ terms}
  \end{subfigure}\hfill
  \begin{subfigure}[b]{0.24\linewidth}
    \includegraphics[height=2.8cm]{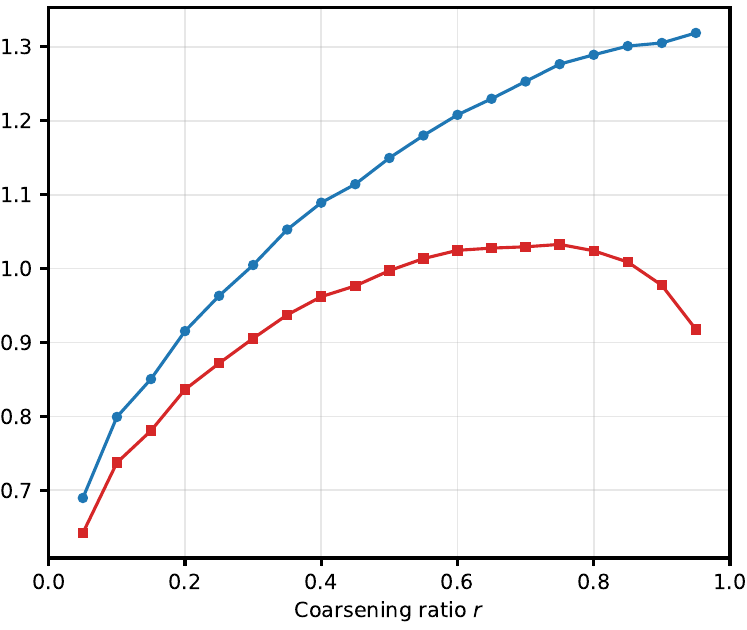}
    \caption*{Theoretical bound} 
  \end{subfigure}\hfill
  \begin{subfigure}[b]{0.24\linewidth}
    \includegraphics[height=2.8cm]{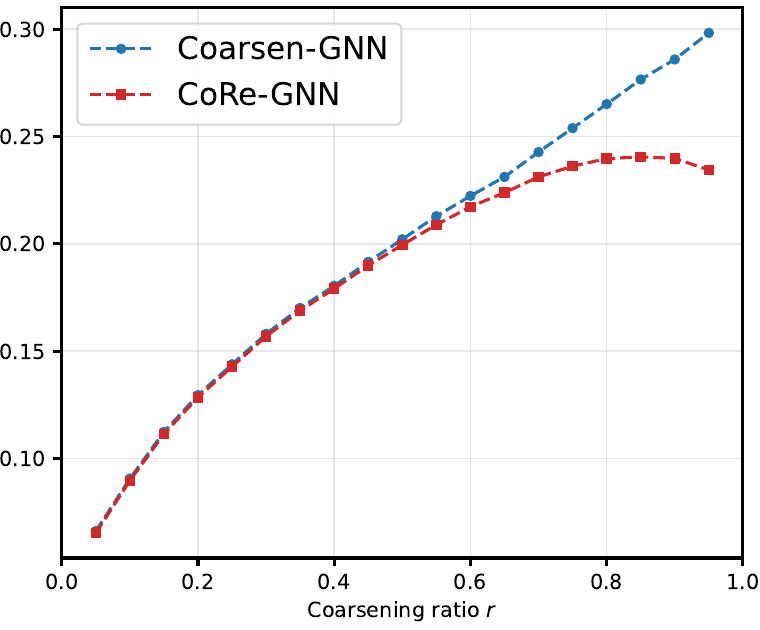}
    \caption*{mean propagation error}
  \end{subfigure}
  \caption{Comparison of the Coarsen-GNN bound (Theorem~\ref{thm:propag}, \textcolor{blue}{blue}) and the CoRe-GNN bound (Theorem~\ref{thm:propagcoregnn}, \textcolor{red}{red}) as a function of the coarsening ratio $r$.
  From left to right: $C_\Pi$, $C_S$ terms, full theoretical bound, and Empirical propagation error.
  SBM parameters 
  are detailed in Appendix~\ref{app:constants}.}
  \label{fig:bound_comparison_sbm}
 \end{figure}

\section{Experimental results}\label{sec:expe}

\paragraph{Setup.}
CoRe-GNN and its batched variant are designed for node classification on medium to large-scale graphs. All experiments are run on a single NVIDIA A40 GPU for reproducibility, and out-of-memory (OOM) results correspond to the memory limit of this hardware. All results are averaged over 10 runs. 

\textbf{(i) Coarsening algorithms:} we consider three possible algorithms to partition the graph into supernodes/clusters. 
\textbf{METIS}~\cite{karypis1997metis} minimizes the normalized cut between clusters and produces balanced partitions with a very efficient implementation. \textbf{Graclus}~\cite{dhillon2007weighted} minimizes the ratio cut via a multilevel approach and is applied iteratively to approach the target coarsening ratio. \textbf{Loukas}~\cite{loukas2019graph} is designed to minimize the RSA for the self-loop normalized Laplacian, as implemented in~\cite{joly2024graph, joly2025taxonomy}.  
We report results for the best-performing algorithm \emph{per dataset} and \emph{per model}. Pseudo-codes and implementation details are provided in Appendix~\ref{app:coarsening_algos}. 
The coarsening ratio $r$ is also a tunable hyperparameter, constrained by available GPU memory. In contrast to Coarsen-GNN, increasing $r$ does not systematically reduce the computational cost of CoRe-GNN. 
While larger values of $r$ reduce the size of the coarsened graph, they also lead to larger clusters, increasing the number of intra-cluster connections and potentially causing memory issues.
This reflects a trade-off in $r$, as further detailed in Appendix~\ref{app:efficiency_regimes}.
For all reported results, we select the best coarsening ratio $r$ per model and dataset.

\textbf{(ii) Models.} We evaluate CoRe-GNN and its batched variant CoRe-GNN(b). As CoRe-GNN approximates the full GCN propagation, we compare against full GCN~\cite{kipf2016GCN} when memory allows. We also compare against our implementation of Cluster-GCN, which performs only intra-cluster propagation and Coarsen-GNN~\cite{joly2024graph, joly2025taxonomy}, which trains and infers entirely on the coarsened graph.
Our goal is to evaluate the effect of the propagation structure rather than to maximize benchmark scores, and we therefore do not include architecture-specific tricks such as residual connections. 

\textbf{(iii) Hyperparameters.} All models are trained with the Adam optimizer for up to 800 epochs with a ReduceLROnPlateau scheduler and a dropout of $0.5$. Each model is evaluated over a hyperparameter grid covering hidden dimension, number of layers, learning rate, weight decay, and coarsening ratio where applicable. The grid varies per model and dataset group. Full hyperparameter grids and per-dataset selected values are provided in Appendix~\ref{app:hyperparameters}.

\begin{wrapfigure}{r}{0.37\textwidth}
\centering
\vspace{-2.7em}
\includegraphics[width=\linewidth]{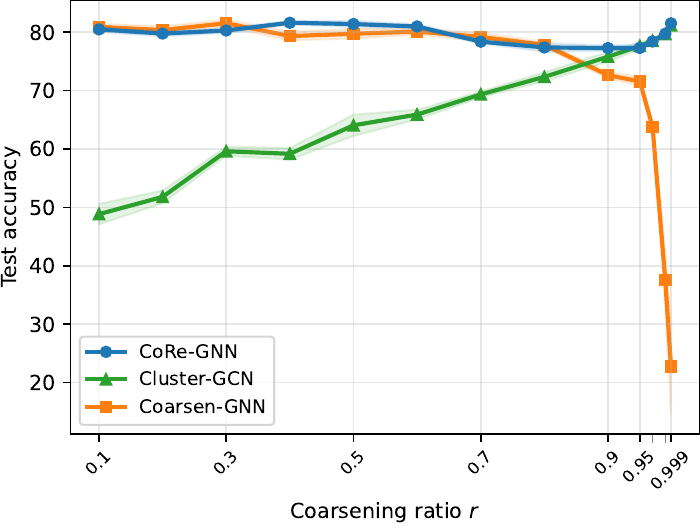}
\caption{Accuracy as a function of the coarsening ratio $r$ on Cora, with fixed hyperparameters ($K=2$, $d=256$, lr=$0.01$, wd=$0.0005$, Loukas coarsening algorithm).}
\label{fig:ablation_r_cora}
\vspace{-1em}
\end{wrapfigure}

\paragraph{Influence of coarsening ratio.} While in the main tables we select the best coarsening ratio for each model, Figure~\ref{fig:ablation_r_cora} reports accuracy as a function of $r$ on Cora with fixed hyperparameters. At low ratios, clusters are small as there are many super-nodes, and Coarsen-GNN preserves per-node discriminability well, peaking at $r=0.3$ (max $81.53\%$). At higher coarsening ratios, accuracy drops as more nodes share the same prediction,
although these regimes are precisely those where coarsening is computationally most beneficial. 
In contrast, Cluster-GCN performs poorly at low ratios, where clusters are too small for effective intra-cluster propagation, and improves as $r$ increases. However, at very high ratios, each cluster approaches the full graph, which removes the benefit of batching and can lead to memory issues. Peak performance is reached at $r=0.999$ ($81.17\%$).
CoRe-GNN's strong performance across all ratios, with a peak accuracy of $81.62\%$ at $r=0.4$, suggests that combining both propagation mechanisms mitigates this trade-off.

\paragraph{Homophilic datasets.} We evaluate on the three Planetoid citation networks~\cite{yang2016revisiting}: Cora, Citeseer and Pubmed. We restrict to the main connected component of Cora and Citeseer, which may explain small differences with other benchmarks. On Cora and Citeseer, Coarsen-GNN achieves the best mean accuracy, while CoRe-GNN is a close second. Both outperform full GCN, as coarsening preserves low-frequency graph structure that aligns well with homophilic label distributions. On Pubmed, CoRe-GNN ranks first. These datasets are not particularly challenging for coarsening-based methods, as clusters tend to be label-homogeneous.

\input{tables/planetoid_heterophilic_new}

\paragraph{Heterophilic datasets.} We consider Chameleon, Squirrel~\cite{rozemberczki2021multi} and Amazon Ratings~\cite{platonovcritical}. On these datasets, Coarsen-GNN is limited by its constant-by-block predictions: as most coarsening algorithms merge neighboring nodes, nodes within the same supernode often carry different labels, directly capping accuracy. CoRe-GNN significantly outperforms Coarsen-GNN on all three datasets. As shown in recent benchmarks~\cite{luo2024classic}, heterophilic graphs require long-range information to be classified accurately, which explains why CoRe-GNN also outperforms Cluster-GCN. On Squirrel, full GCN outperforms CoRe-GNN, while CoRe-GNN(b) achieves the best results, suggesting that the batching scheme provides additional regularization on this dataset.

\paragraph{Large-scale graphs.} We report results on ogbn-arxiv~\cite{hu2020open}, Reddit~\cite{hamilton2017inductive} and ogbn-products~\cite{hu2020open}. On ogbn-arxiv, Coarsen-GNN performs poorly as the graph has weak community structure, while CoRe-GNN and Cluster-GCN are both competitive but do not surpass full GCN. On Reddit and ogbn-products, full GCN and CoRe-GNN without batching are out of memory. CoRe-GNN(b) significantly outperforms both Cluster-GCN and Coarsen-GNN on these datasets, suggesting that the inter-cluster and intra-cluster propagations capture complementary information. 

\input{tables/large_city_new}
\paragraph{Long-range tasks.} We evaluate on the Paris and London road networks~\cite{liang2025towards}, two recently introduced large graphs specifically designed to measure long-range interactions in graph machine learning. A description of all datasets is provided in Appendix~\ref{app:datasets}. These graphs require aggregating information over many hops to correctly classify nodes, making them particularly challenging for local message passing methods.
Among the baselines, Coarsen-GNN outperforms both Cluster-GCN and full GCN as its
inter-cluster propagation already captures some long-range structure. CoRe-GNN and
CoRe-GNN(b) achieve the best accuracy among the architectures considered here,
indicating that combining inter- and intra-cluster propagations provides the most
effective long-range aggregation in this setting.

\section{Conclusion}

CoRe-GNN combines inter-cluster and intra-cluster propagations in parallel at each layer, overcoming the limitations of both graph coarsening and Cluster-GCN. The architecture is agnostic to the choice of coarsening algorithm, admits a natural batching scheme that scales to graphs with millions of nodes, and inherits provable propagation guarantees from graph coarsening. On node classification benchmarks spanning homophilic, heterophilic, large-scale, and long-range graphs, CoRe-GNN consistently outperforms both Coarsen-GNN and Cluster-GCN baselines, and achieves substantial gains over full GCN on long-range tasks on the cities datasets.

\paragraph{Limitations and outlook.} The unified matrix formulation underlying CoRe-GNN applies to GNNs whose forward pass can be expressed through a fixed propagation matrix $\matrixpropag$, excluding architectures with residual connections or learned propagation matrices such as Graph Attention Networks~\cite{velivckovic2017graph}. Adapting CoRe-GNN to these settings is a natural direction for future work, alongside using separate optimizers for the two branches and designing residual connections compatible with the two-branch structure. On the theoretical side, a multi-level extension where the coarsened graph is itself coarsened recursively would naturally generalize the architecture.

Beyond these extensions, CoRe-GNN (batched) also suggests a natural connection to asynchronous federated learning.  
The coarsened graph plays the role of a central aggregator that receives local updates from the clusters selected in the current batch. For clusters that have not yet reported back, the central aggregator sends a self-loop proxy in place of their local update. This asynchronous communication pattern suggests that CoRe-GNN could serve as a framework for federated GNN training on large graphs, where not all agents are available at each round.

\section*{Acknowledgments}
The authors acknowledge the fundings of France 2030, PEPR IA, ANR-23-PEIA-0008 and European
Union ERC-2024-STG-101163069 MALAGA.

\begingroup
\footnotesize

\bibliography{refs}  
\endgroup


\appendix

\section{Proofs}

\subsection{Proof of Theorem~\ref{thm:propagcoregnn}}
\label{app:proof_coregnn}

\begin{align}
\lVert Sx - (\sintra x + Q\sextrac Px) \rVert_L 
&= \lVert (\sintra + \sextra) x - \sintra x - Q P \sextra Q P x \rVert_L \\
&= \lVert \sextra x - \Pi \sextra \Pi x \rVert_L \\
&\leq \epsilon_{L,Q,P,\preservedspace} \lVert x \rVert_L \left( C_S' + C_\Pi' \right)
\end{align}
where the second equality uses $Q\sextrac P = \Pi\sextra\Pi$, and the last inequality applies Theorem~\ref{thm:propag} to $\sextra$ in place of $S$. Note that the assumptions and constants are now stated for $\sextra$: $\Pi$ and $\sextra$ must be $\ker(L)$-preserving, $\sextra$ must be $\preservedspace$-preserving, and the constants $C_S' = \lVert \sextra \rVert_L$ and $C_\Pi' = \lVert \Pi\sextra \rVert_L$ depend on $\sextra$ rather than $S$.

\subsection{Proof commutativity of the lifting processes}
\begin{lemma}[Commutativity of $Q$ with $\sigma$]
\label{lem:commute}
Let $Q \in \mathbb{R}^{N \times n}$ be a well-partitioned matrix with non-negative entries, and let $\sigma$ be a positively homogeneous pointwise activation function, i.e., $\sigma(\alpha x) = \alpha\,\sigma(x)$ for all $\alpha > 0$. Then for any $X \in \mathbb{R}^{n \times d}$:
\[
Q\,\sigma(X) = \sigma(QX)
\]
\begin{proof}
Since $Q$ is well-partitioned, for every row $i$ there exists a unique $k^* \in [n]$ such that $Q_{ik^*} > 0$ and $Q_{ij} = 0$ for all $j \neq k^*$. Therefore:
\[
[Q\,\sigma(X)]_i = \sum_{j=1}^n Q_{ij}\,\sigma(X_j) = Q_{ik^*}\,\sigma(X_{k^*})
\]
and
\[
[\sigma(QX)]_i = \sigma\!\left(\sum_{j=1}^n Q_{ij} X_j\right) = \sigma(Q_{ik^*} X_{k^*}) = Q_{ik^*}\,\sigma(X_{k^*})
\]

\end{proof}
\end{lemma}

\begin{remark}
The homogeneity condition $\sigma(\alpha x) = \alpha\,\sigma(x)$ for $\alpha > 0$ is satisfied by the identity function (as in SGC~\cite{wu2019SGC}) and by ReLU (as in classical GCN~\cite{kipf2016GCN}), but not by sigmoid or tanh activations.
\end{remark}

\subsection{Relation between CoRe-GNN and Coarsen-GNN}
\label{app:generalize}

We denote by $C_k = \{i \in [N] : [Q]_{ik} > 0\}$ the set of original nodes mapped to supernode $k$, for $k = 1, \ldots, n$. We first establish the structure of $P\sintra Q$ and then show that $\Pi \sintra \Pi$ is block-diagonal by cluster.

\begin{lemma}[$P\sintra Q$ is diagonal and $\Pi\sintra\Pi$ is block-diagonal]
\label{lem:block_diagonal}
Let $Q$ be a well-partitioned lifting matrix and $P=Q^+$ the moore penrose inverse of this matrix  with the same support as $Q^\top$. Then $P\sintra Q \in \mathbb{R}^{n\times n}$ is diagonal with entries:
\[
[P\sintra Q]_{kk} = \sum_{u,v \in C_k} [P]_{ku} [\sintra]_{uv} [Q]_{vk}
\]
and $\Pi\sintra\Pi = Q(P\sintra Q)P$ is block-diagonal in $\mathbb{R}^{N\times N}$, with one block per cluster.
\begin{proof}
For $k \neq l$ with $k, l \in [n]$:
\begin{align}
[P\sintra Q]_{kl} 
&= \sum_{u,v=1}^N [P]_{ku} [\sintra]_{uv} [Q]_{vl} \\
&= \sum_{u \in C_k} \sum_{v \in C_l} [P]_{ku} [\sintra]_{uv} [Q]_{vl} \\
&= 0
\end{align}
since $[\sintra]_{uv} = 0$ for $u \in C_k$, $v \in C_l$ with $k \neq l$. 

For $i \in C_k$ and $j \in C_l$ with $k \neq l$:
\begin{align}
[\Pi\sintra\Pi]_{ij} 
&= \sum_{k'=1}^n [Q]_{ik'} [P\sintra Q]_{k'k'} [P]_{k'j} \\
&= [Q]_{ik} [P\sintra Q]_{kk} [P]_{kj} = 0
\end{align}
since $[P]_{kj} = 0$ for $j \notin C_k$ as $P=Q^+$ has the same support as $Q^\top$.

 and, for $i, j \in C_k$, $[\Pi\sintra\Pi]_{ij}$ is in general nonzero.
\end{proof}
\end{lemma}

Since $\sextrac = P\sextra Q = PSQ - P\sintra Q$, the term $P\sintra Q$ is precisely what distinguishes $\newmatrixpropag$ from $\sextrac$: it acts as a diagonal self-loop correction on the coarsened nodes, absent from CoRe-GNN. By contrast, $\Pi\sextra\Pi$ has nonzero entries across different clusters, capturing inter-cluster dependencies by blocks. 

\section{Analysis of Propagation Bound Constants}
\label{app:constants}

This section provides a deeper analysis of the constants $C_S$, $C_\Pi$, $C_S'$ and $C_\Pi'$ appearing in Theorems~\ref{thm:propag} and~\ref{thm:propagcoregnn}. For each graph, we plot the four quantities: $C_S$ and $C_\Pi$ terms, full theoretical upper bound, and mean propagation error averaged over 100 smooth signals sampled from the low-frequency subspace of $L$, all as a function of the coarsening ratio $r$, comparing Coarsen-GNN (blue) and CoRe-GNN (red). We first analyze four SBM graphs with varying community structure (Section~\ref{app:sbm_analysis}), then four real datasets (Section~\ref{app:real_analysis}), before discussing the results(Section~\ref{app:constants_discussion}).

\subsection{Empirical analysis on SBM graphs}
\label{app:sbm_analysis}

We analyze the bounds on four synthetic Stochastic Block Model (SBM) graphs, each with 5 communities of 100 nodes. In an SBM, edges are drawn independently with probability $p_{\rm in}$ within communities and $p_{\rm out}$ across communities. We vary $p_{\rm in}$ and $p_{\rm out}$ to span a range of spectral connectivities. SBM-1 and SBM-2 are sparse graphs with low $p_{\rm in}$, while SBM-3 and SBM-4 are denser with higher inter-community connectivity. All graphs are coarsened with the Loukas algorithm. For each graph we show its structure and the four bound quantities as a function of the coarsening ratio $r$. SBM-2 corresponds to the graph shown in the main paper (Figure~\ref{fig:bound_comparison_sbm}).

\begin{figure}[H]
\centering
\begin{subfigure}[b]{0.18\linewidth}
    \includegraphics[width=\linewidth]{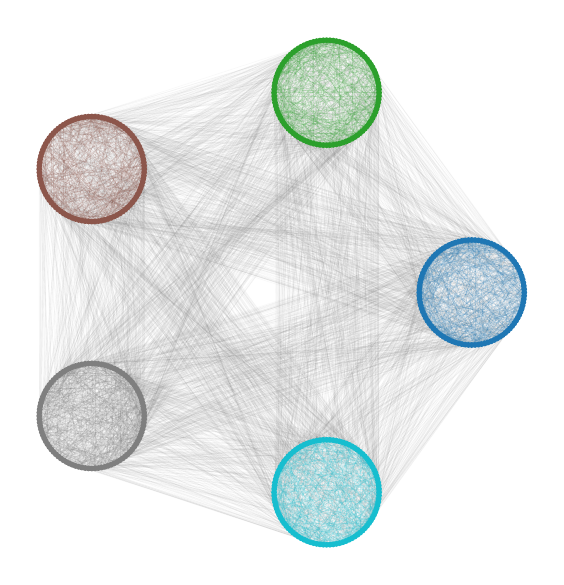}
    \caption*{Graph structure}
\end{subfigure}\hfill
\begin{subfigure}[b]{0.18\linewidth}
    \includegraphics[height=2.1cm]{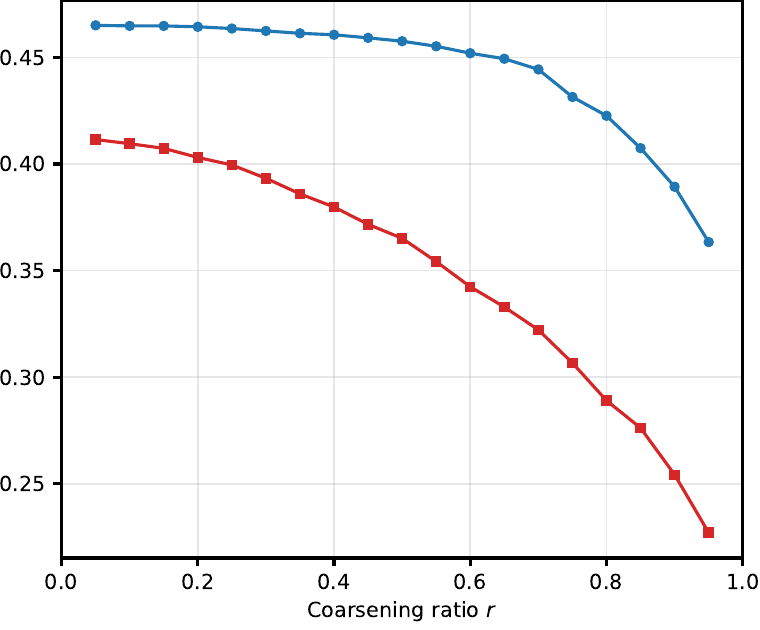}
    \caption*{$C_\Pi$ terms}
\end{subfigure}\hfill
\begin{subfigure}[b]{0.18\linewidth}
    \includegraphics[height=2.1cm]{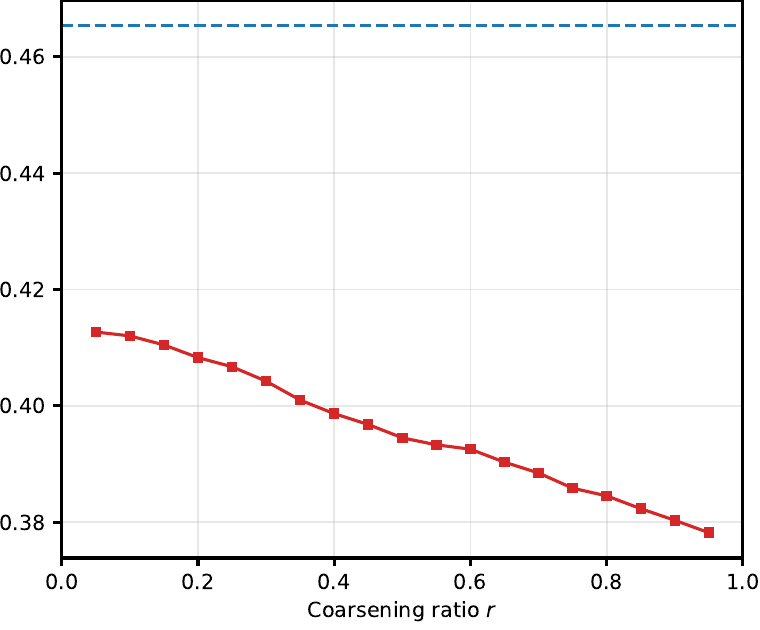}
    \caption*{$C_S$ terms}
\end{subfigure}\hfill
\begin{subfigure}[b]{0.18\linewidth}
    \includegraphics[height=2.1cm]{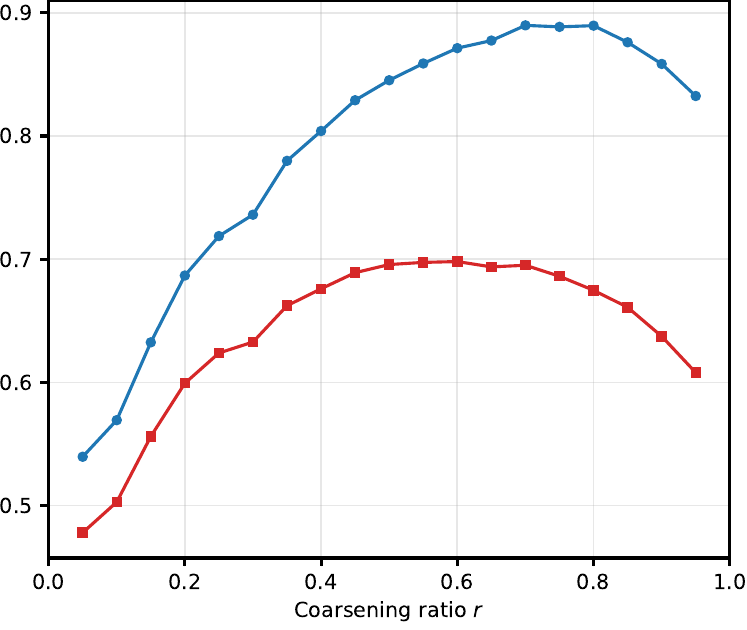}
    \caption*{Theoretical bound}
\end{subfigure}\hfill
\begin{subfigure}[b]{0.18\linewidth}
    \includegraphics[height=2.1cm]{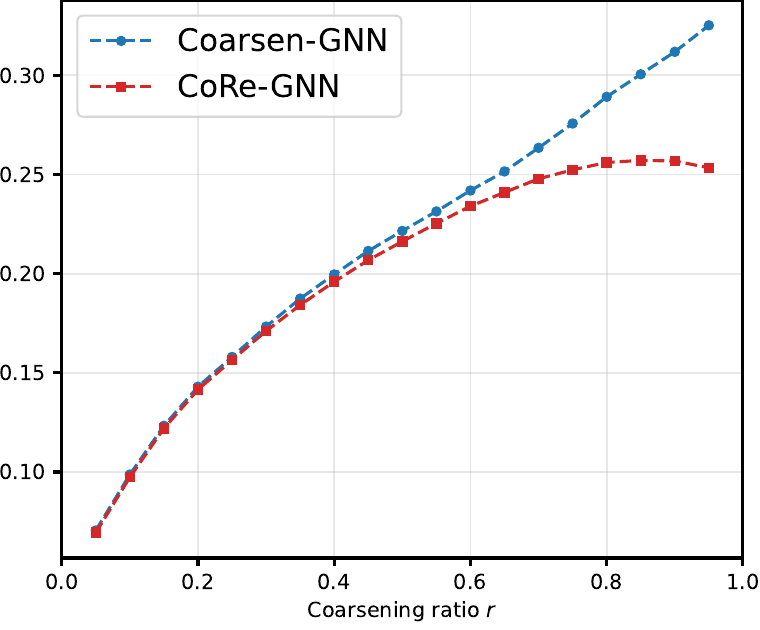}
    \caption*{propagation error}
\end{subfigure}
\caption{SBM-1 ($p_{\rm in} = 0.08$, $p_{\rm out} = 0.030$)}
\end{figure}

\begin{figure}[H]
\centering
\begin{subfigure}[b]{0.18\linewidth}
    \includegraphics[width=\linewidth]{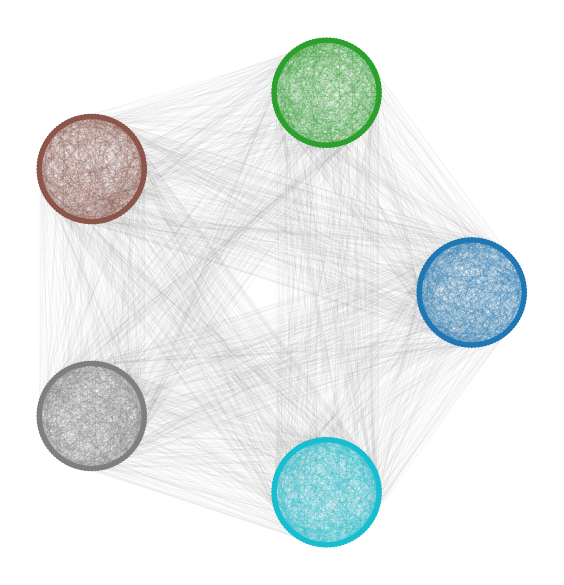}
    \caption*{Graph structure}
\end{subfigure}\hfill
\begin{subfigure}[b]{0.18\linewidth}
    \includegraphics[height=2.1cm]{pdf_figures/upper-bound/sbm_2/SBM_2_C_Pi_terms.pdf}
    \caption*{$C_\Pi$ terms}
\end{subfigure}\hfill
\begin{subfigure}[b]{0.18\linewidth}
    \includegraphics[height=2.1cm]{pdf_figures/upper-bound/sbm_2/SBM_2_C_S_terms.pdf}
    \caption*{$C_S$ terms}
\end{subfigure}\hfill
\begin{subfigure}[b]{0.18\linewidth}
    \includegraphics[height=2.1cm]{pdf_figures/upper-bound/sbm_2/SBM_2_RHS_bound.pdf}
    \caption*{Theoretical bound}
\end{subfigure}\hfill
\begin{subfigure}[b]{0.18\linewidth}
    \includegraphics[height=2.1cm]{pdf_figures/upper-bound/sbm_2/SBM_2_LHS.pdf}
    \caption*{propagation error}
\end{subfigure}
\caption{SBM-2 ($p_{\rm in} = 0.15$, $p_{\rm out} = 0.020$)}
\end{figure}

\begin{figure}[H]
\centering
\begin{subfigure}[b]{0.18\linewidth}
    \includegraphics[width=\linewidth]{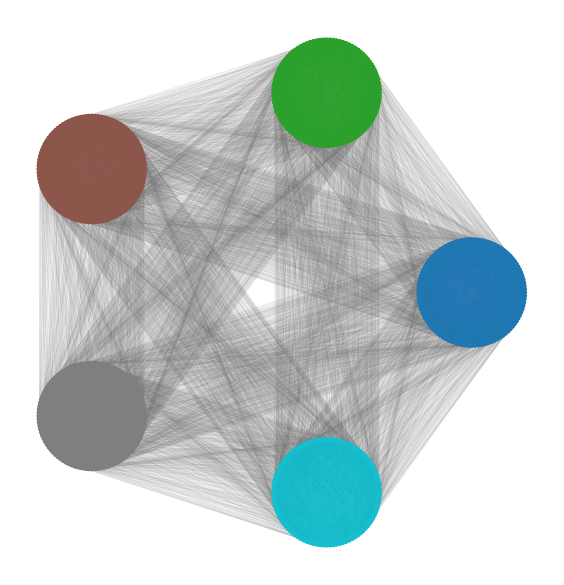}
    \caption*{Graph structure}
\end{subfigure}\hfill
\begin{subfigure}[b]{0.18\linewidth}
    \includegraphics[height=2.1cm]{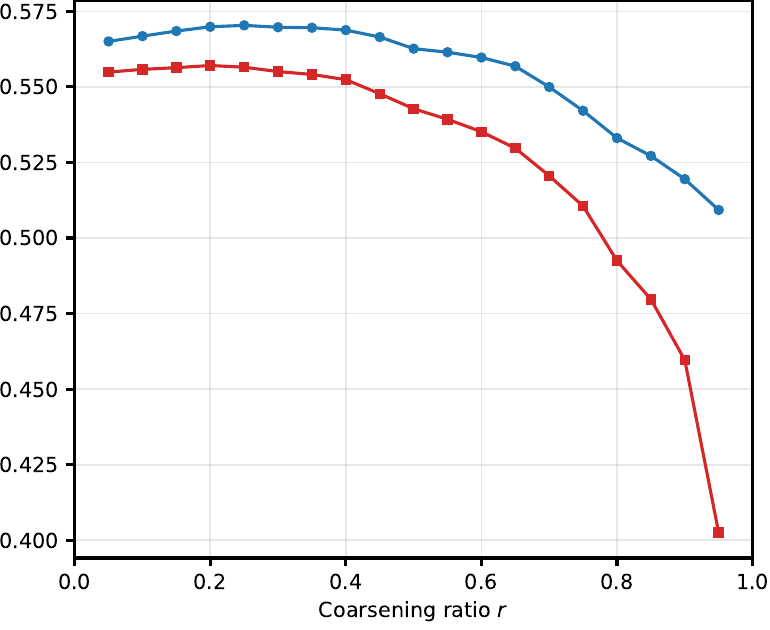}
    \caption*{$C_\Pi$ terms}
\end{subfigure}\hfill
\begin{subfigure}[b]{0.18\linewidth}
    \includegraphics[height=2.1cm]{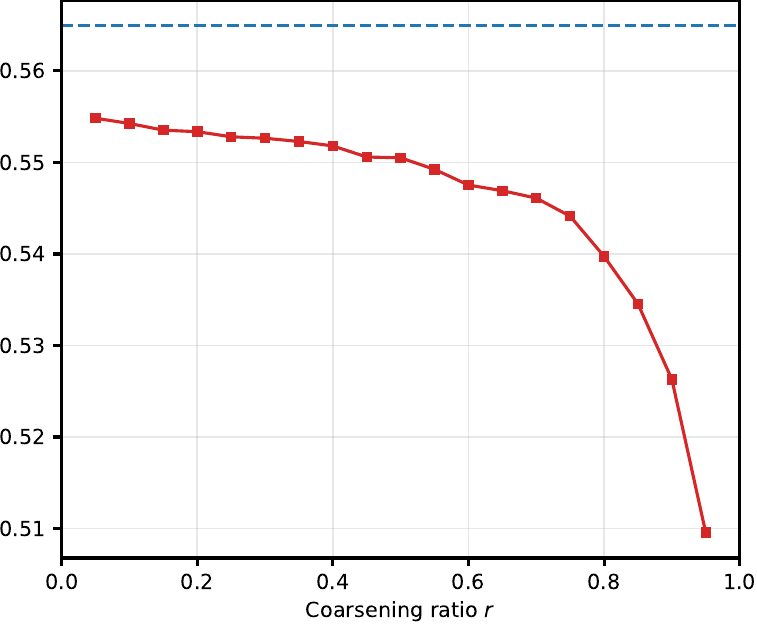}
    \caption*{$C_S$ terms}
\end{subfigure}\hfill
\begin{subfigure}[b]{0.18\linewidth}
    \includegraphics[height=2.1cm]{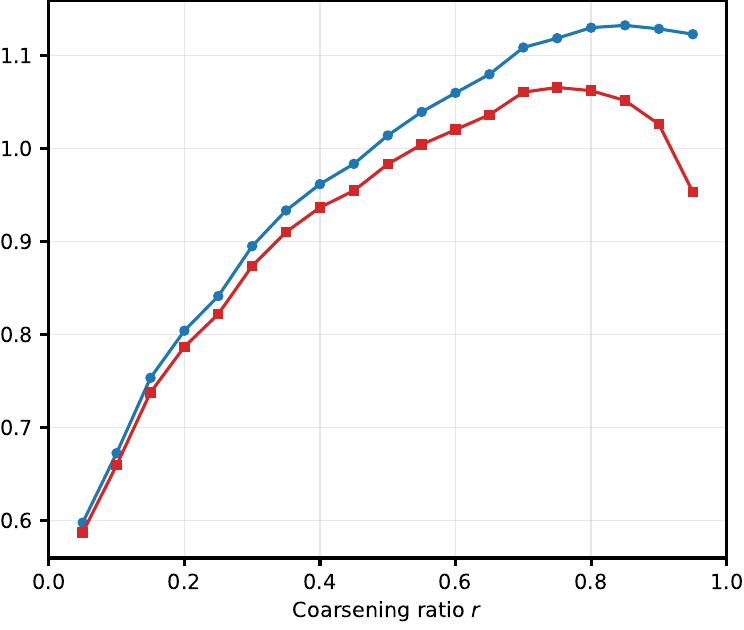}
    \caption*{Theoretical bound}
\end{subfigure}\hfill
\begin{subfigure}[b]{0.18\linewidth}
    \includegraphics[height=2.1cm]{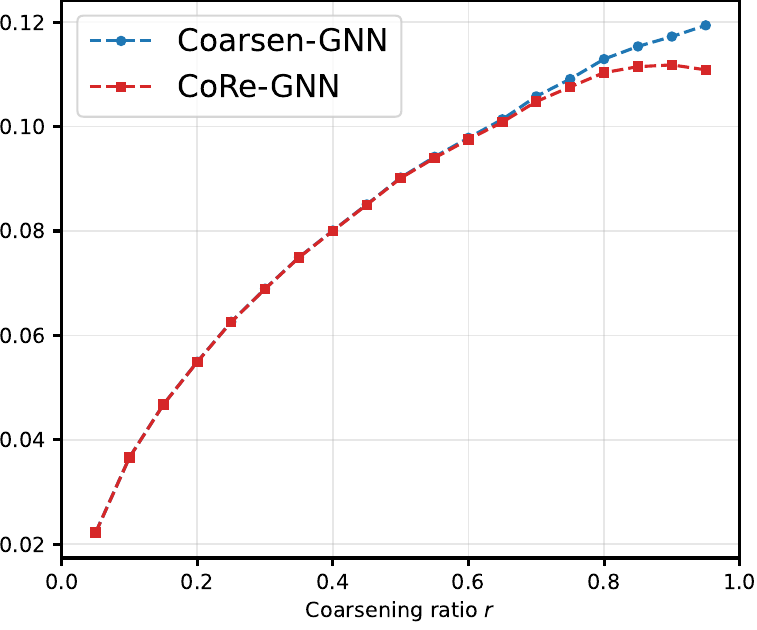}
    \caption*{propagation error}
\end{subfigure}
\caption{SBM-3 ($p_{\rm in} = 0.70$, $p_{\rm out} = 0.10$)}
\end{figure}

\begin{figure}[H]
\centering
\begin{subfigure}[b]{0.18\linewidth}
    \includegraphics[width=\linewidth]{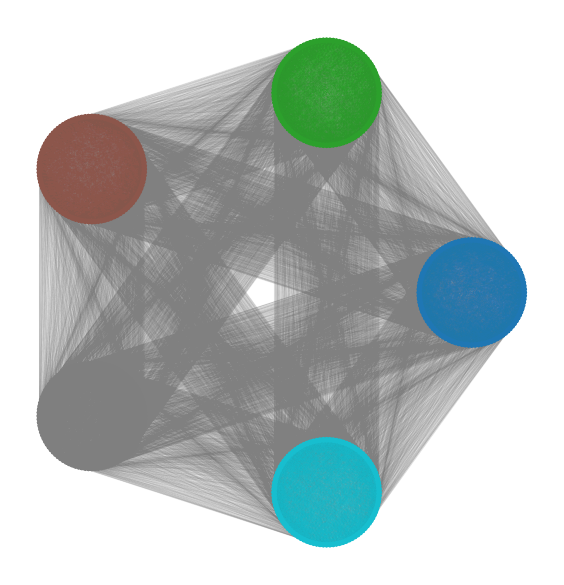}
    \caption*{Graph structure}
\end{subfigure}\hfill
\begin{subfigure}[b]{0.18\linewidth}
    \includegraphics[height=2.1cm]{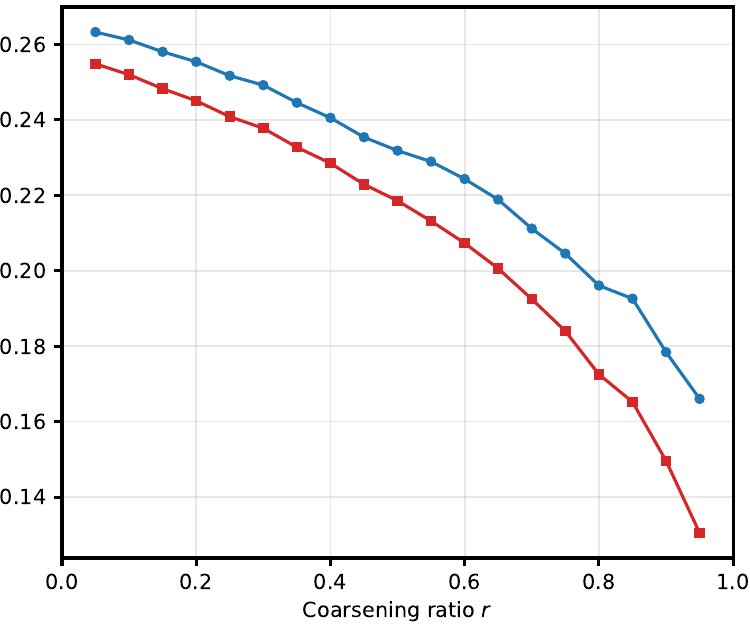}
    \caption*{$C_\Pi$ terms}
\end{subfigure}\hfill
\begin{subfigure}[b]{0.18\linewidth}
    \includegraphics[height=2.1cm]{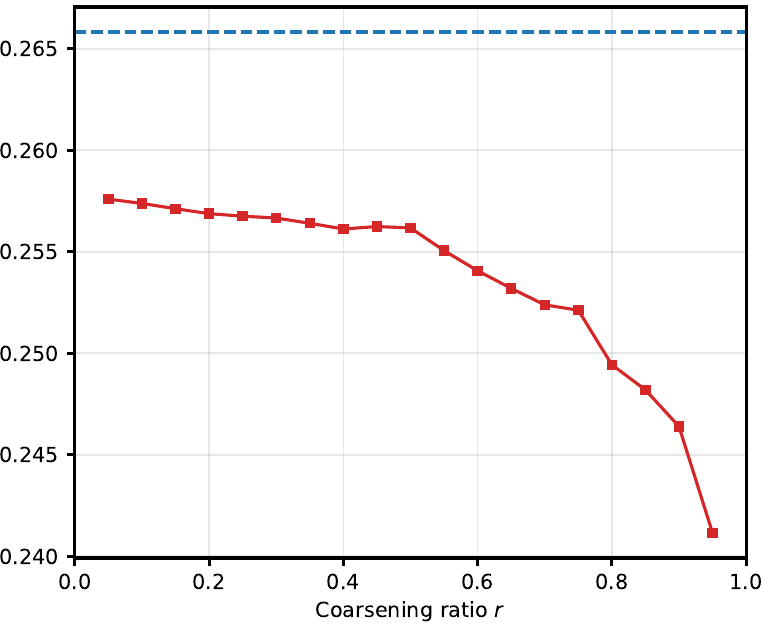}
    \caption*{$C_S$ terms}
\end{subfigure}\hfill
\begin{subfigure}[b]{0.18\linewidth}
    \includegraphics[height=2.1cm]{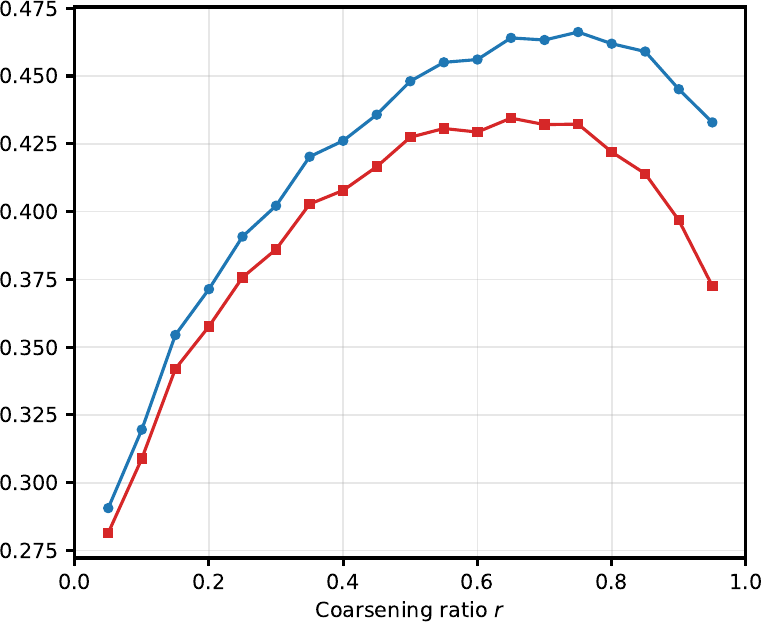}
    \caption*{Theoretical bound}
\end{subfigure}\hfill
\begin{subfigure}[b]{0.18\linewidth}
    \includegraphics[height=2.1cm]{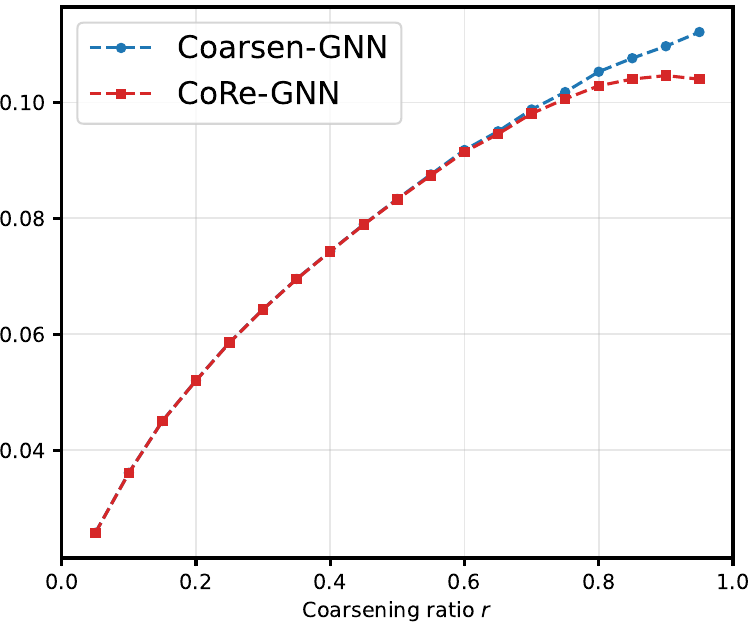}
    \caption*{propagation error}
\end{subfigure}
\caption{SBM-4 ($p_{\rm in} = 0.50$, $p_{\rm out} = 0.20$)}
\end{figure}

\subsection{Empirical analysis on real graphs}
\label{app:real_analysis}
We repeat the same analysis on four real datasets: Cora, Citeseer, Chameleon and Squirrel.

\begin{figure}[H]
\centering
\begin{subfigure}[b]{0.24\linewidth}
    \includegraphics[height=2.8cm]{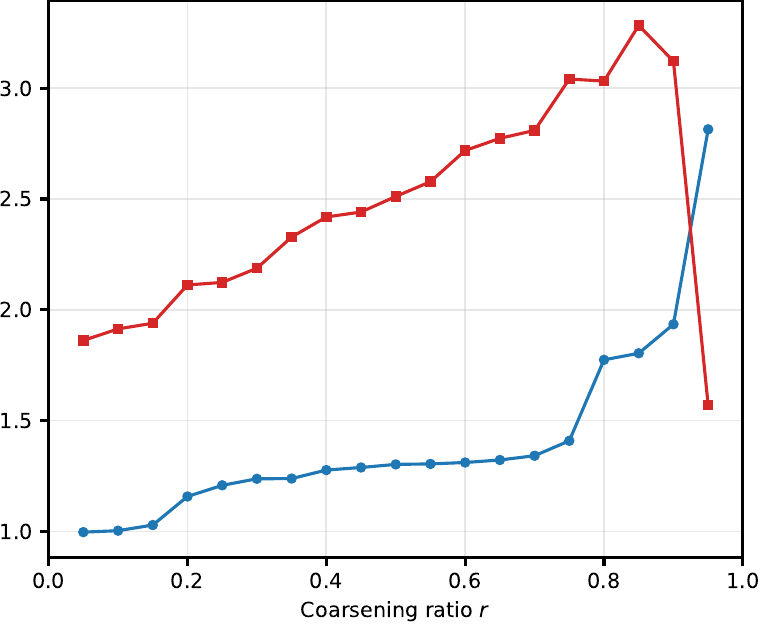}
    \caption*{$C_\Pi$ terms}
\end{subfigure}\hfill
\begin{subfigure}[b]{0.24\linewidth}
    \includegraphics[height=2.8cm]{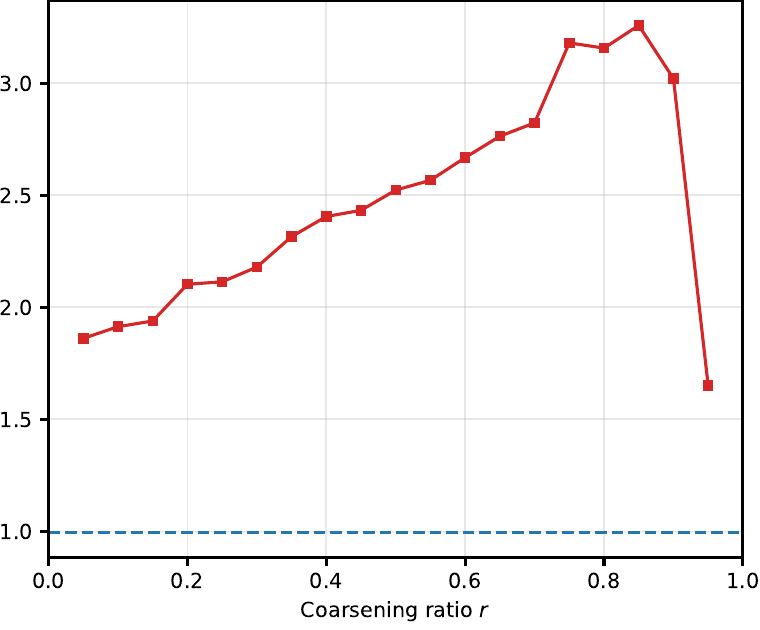}
    \caption*{$C_S$ terms}
\end{subfigure}\hfill
\begin{subfigure}[b]{0.24\linewidth}
    \includegraphics[height=2.8cm]{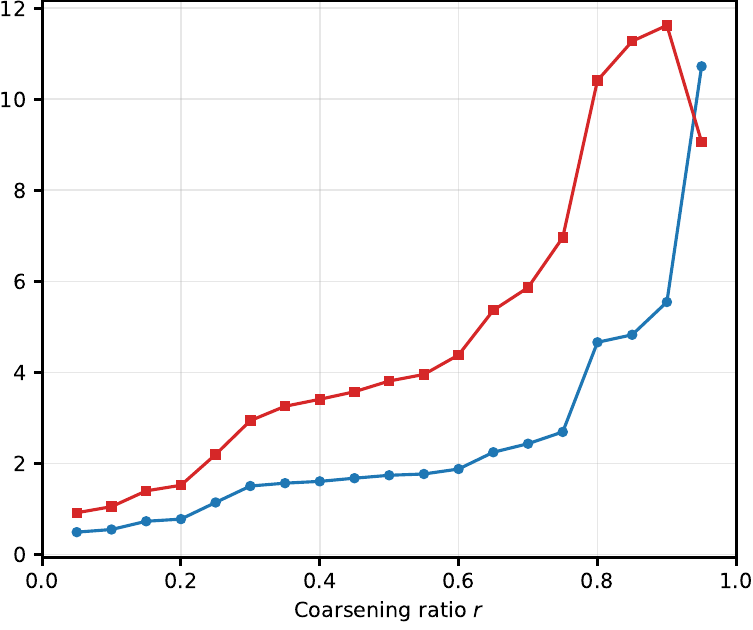}
    \caption*{Theoretical bound}
\end{subfigure}\hfill
\begin{subfigure}[b]{0.24\linewidth}
    \includegraphics[height=2.8cm]{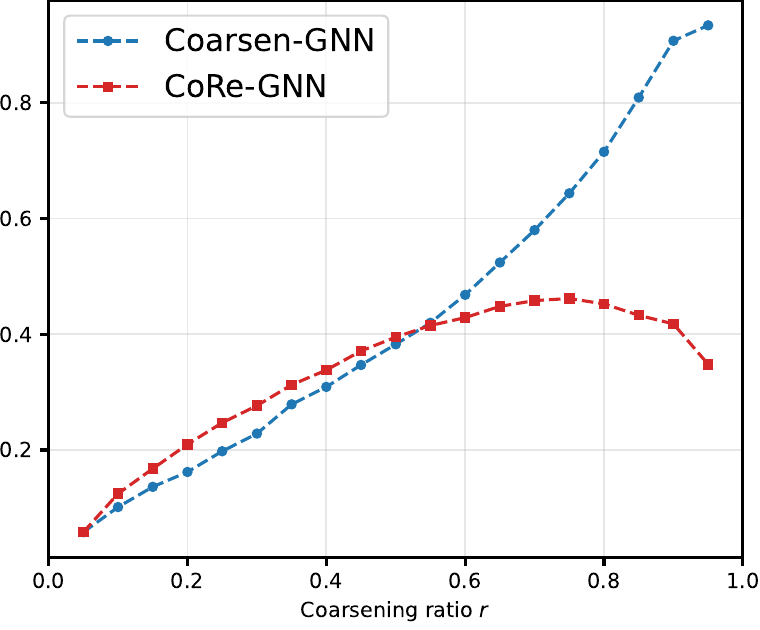}
    \caption*{mean propagation error}
\end{subfigure}
\caption{Cora Dataset}
\end{figure}

\begin{figure}[H]
\centering
\begin{subfigure}[b]{0.24\linewidth}
    \includegraphics[height=2.8cm]{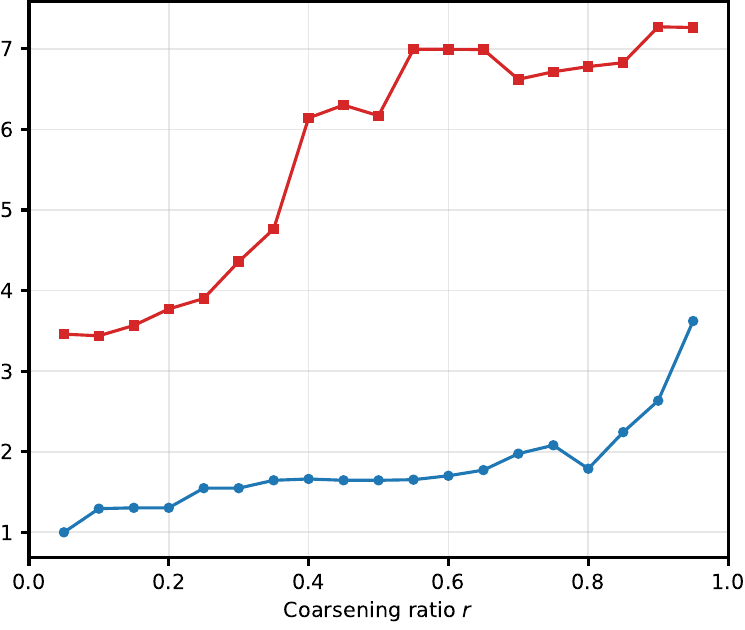}
    \caption*{$C_\Pi$ terms}
\end{subfigure}\hfill
\begin{subfigure}[b]{0.24\linewidth}
    \includegraphics[height=2.8cm]{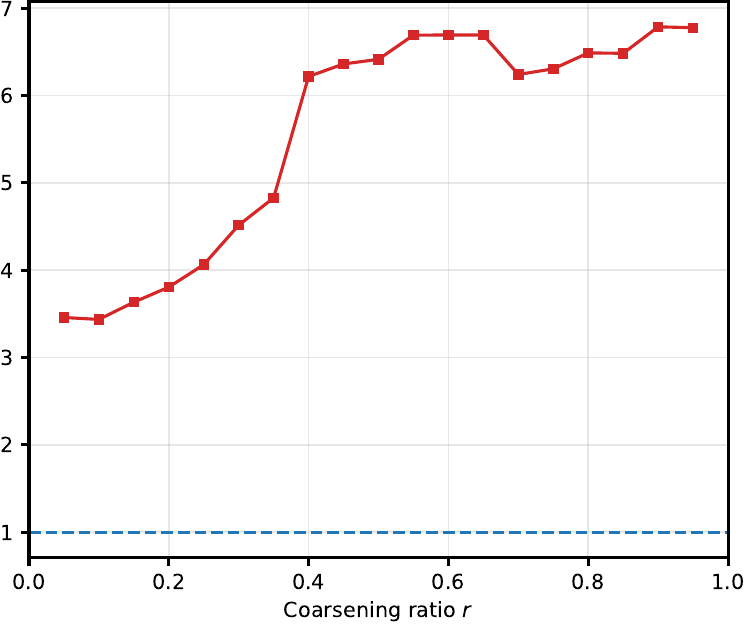}
    \caption*{$C_S$ terms}
\end{subfigure}\hfill
\begin{subfigure}[b]{0.24\linewidth}
    \includegraphics[height=2.8cm]{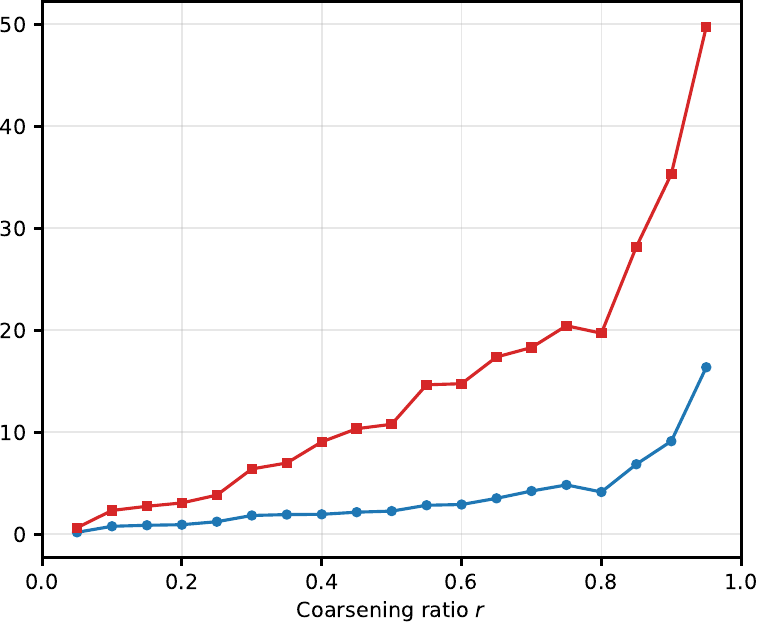}
    \caption*{Theoretical bound}
\end{subfigure}\hfill
\begin{subfigure}[b]{0.24\linewidth}
    \includegraphics[height=2.8cm]{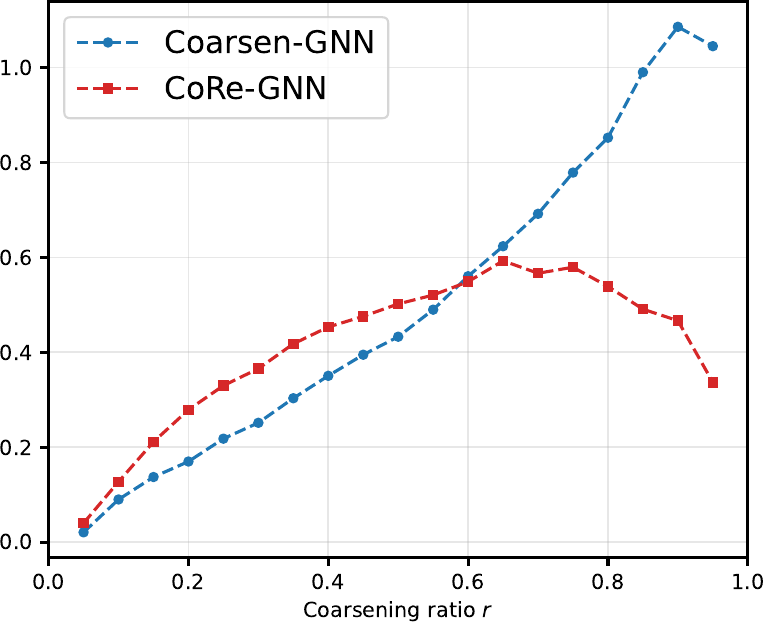}
    \caption*{mean propagation error}
\end{subfigure}
\caption{CiteSeer Dataset}
\end{figure}

\begin{figure}[H]
\centering
\begin{subfigure}[b]{0.24\linewidth}
    \includegraphics[height=2.8cm]{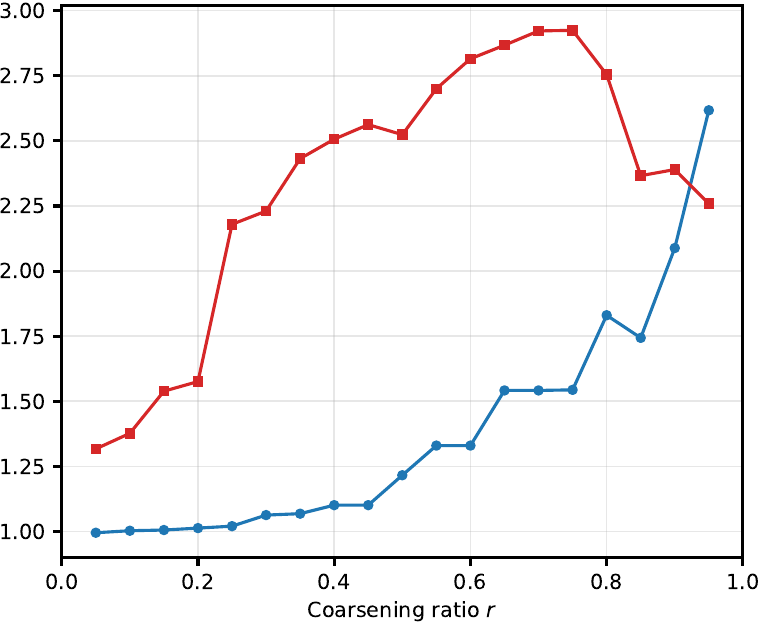}
    \caption*{$C_\Pi$ terms}
\end{subfigure}\hfill
\begin{subfigure}[b]{0.24\linewidth}
    \includegraphics[height=2.8cm]{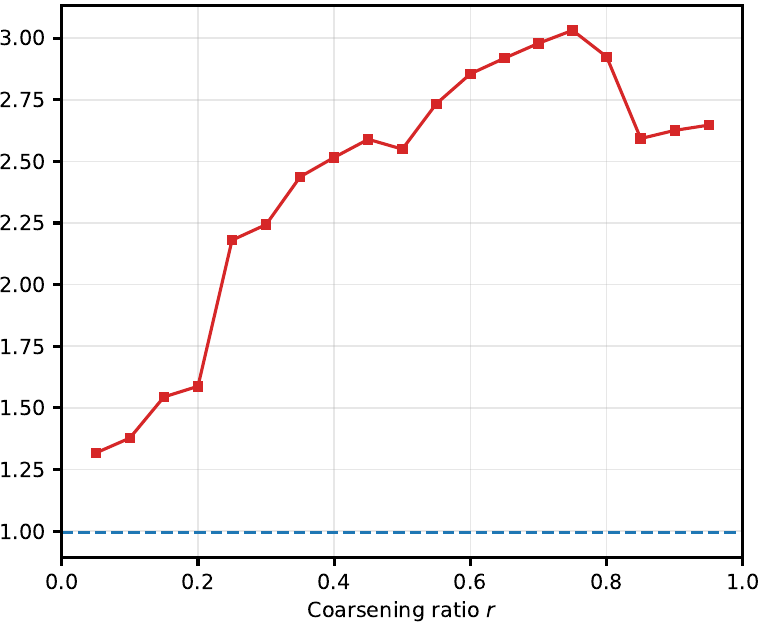}
    \caption*{$C_S$ terms}
\end{subfigure}\hfill
\begin{subfigure}[b]{0.24\linewidth}
    \includegraphics[height=2.8cm]{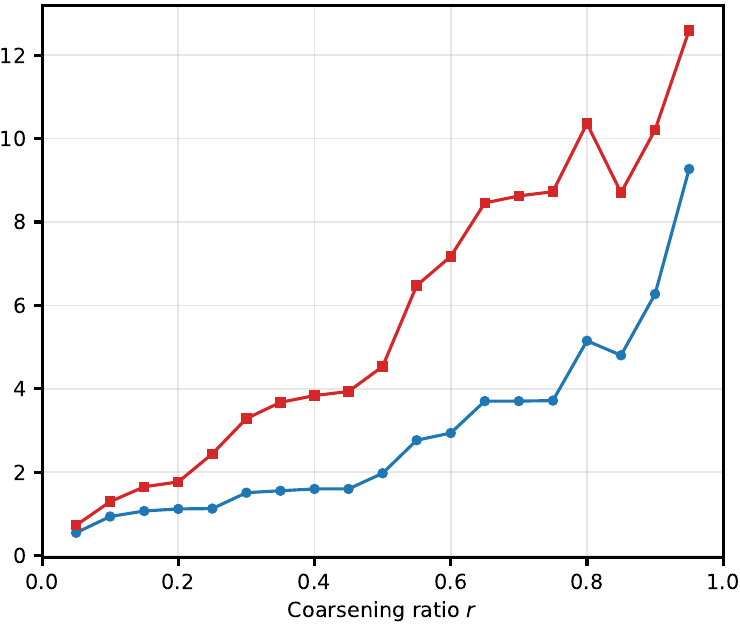}
    \caption*{Theoretical bound}
\end{subfigure}\hfill
\begin{subfigure}[b]{0.24\linewidth}
    \includegraphics[height=2.8cm]{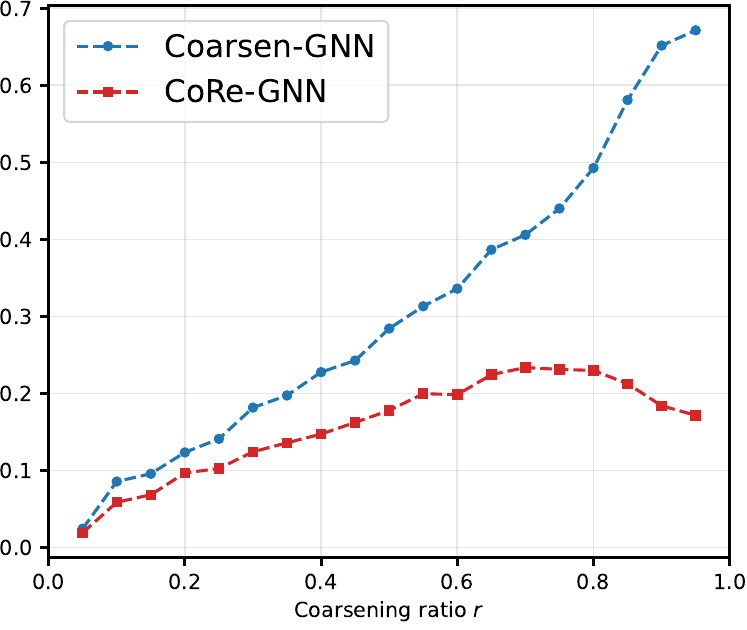}
    \caption*{mean propagation error}
\end{subfigure}
\caption{Chameleon Dataset}
\end{figure}

\begin{figure}[H]
\centering
\begin{subfigure}[b]{0.24\linewidth}
    \includegraphics[height=2.8cm]{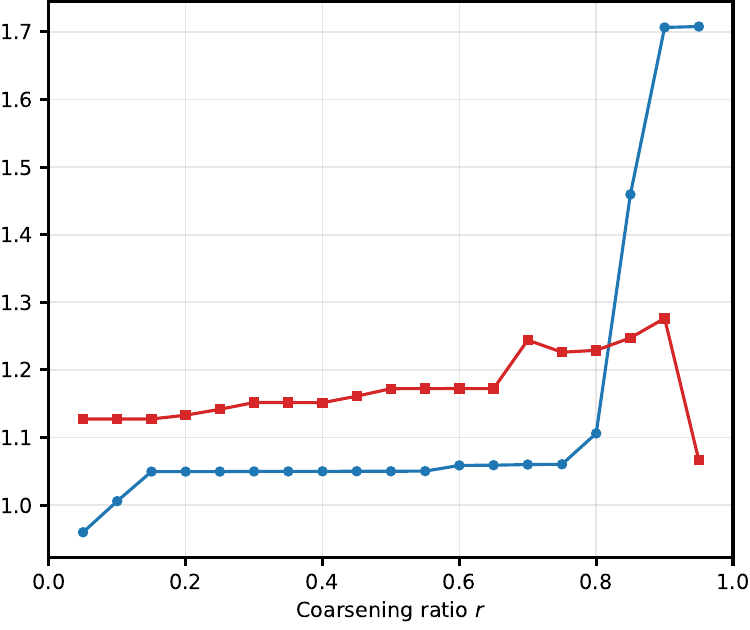}
    \caption*{$C_\Pi$ terms}
\end{subfigure}\hfill
\begin{subfigure}[b]{0.24\linewidth}
    \includegraphics[height=2.8cm]{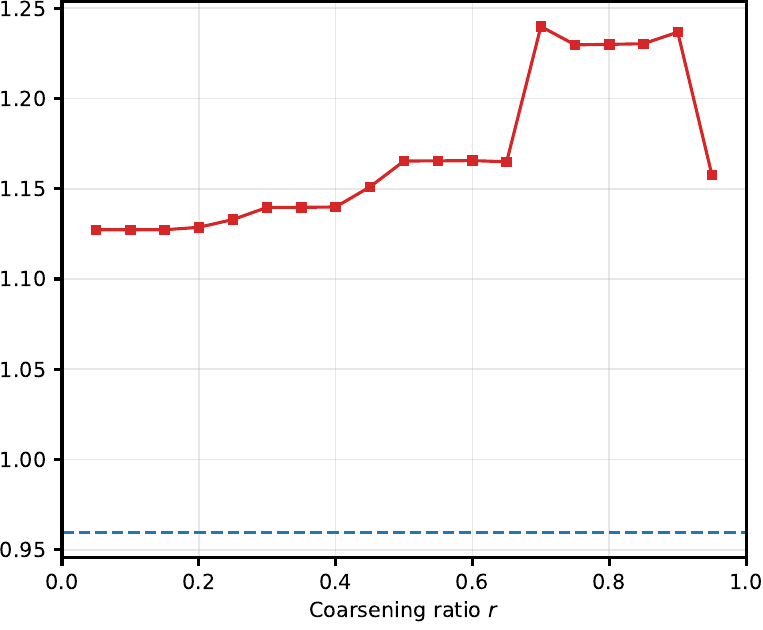}
    \caption*{$C_S$ terms}
\end{subfigure}\hfill
\begin{subfigure}[b]{0.24\linewidth}
    \includegraphics[height=2.8cm]{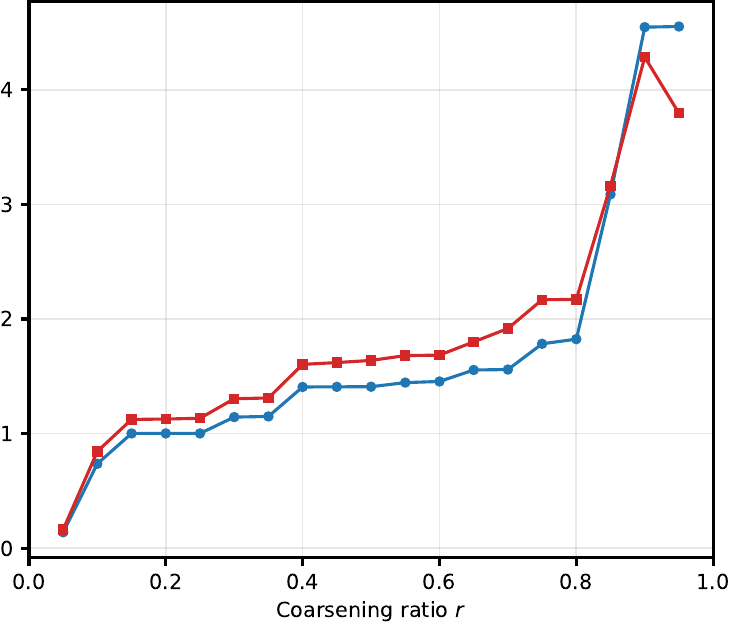}
    \caption*{Theoretical bound}
\end{subfigure}\hfill
\begin{subfigure}[b]{0.24\linewidth}
    \includegraphics[height=2.8cm]{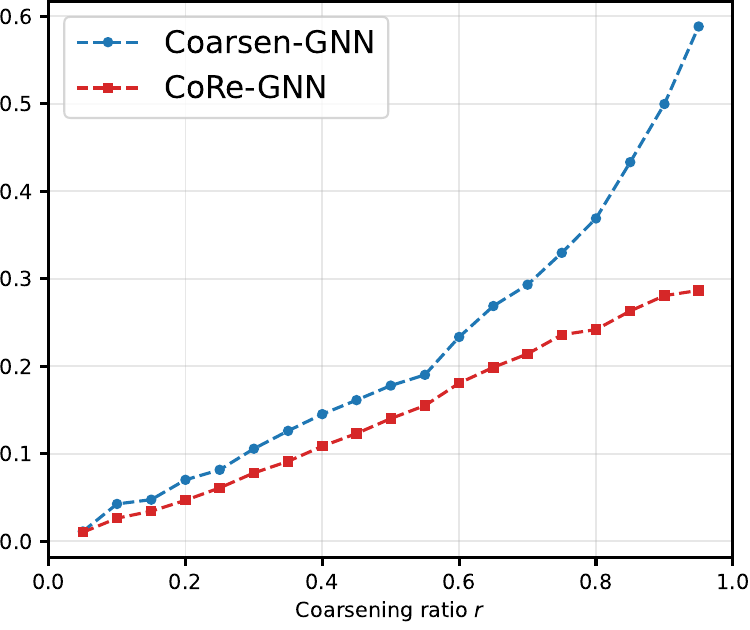}
    \caption*{mean propagation error}
\end{subfigure}
\caption{Squirrel Dataset}
\end{figure}

\subsection{Discussion}
\label{app:constants_discussion}

Across all graphs and coarsening ratios, the mean propagation error averaged over smooth signals is consistently lower for CoRe-GNN than for Coarsen-GNN, even when the theoretical bound constants are not strictly smaller. This gap between the bound and the true error is expected: the bound is a worst-case quantity while the mean error reflects typical behavior on low-frequency signals, which are precisely the signals preserved by spectral coarsening. Heterophily does not affect this analysis, as we sample signals from the low-frequency subspace of $L$ regardless of the graph's label structure.

\section{Detailed Complexity Analysis}
\label{app:complexity}

This section provides a detailed complexity analysis of all methods considered in this paper. Table~\ref{tab:complexity} summarizes time and memory complexity, and Section~\ref{app:efficiency_regimes} analyzes the regimes in which CoRe-GNN is theoretically cheaper than classical GCN.
\subsection{Complexity table}

Table~\ref{tab:complexity} summarizes the time and memory complexity of each method. We use $K$ for the number of layers, $d$ for the hidden dimension, $N$ for the number of nodes, $E = \mathrm{nnz}(\matrixpropag)$ for the number of edges, $n = (1-r)N$ for the number of super-nodes at coarsening ratio $r$, $E_c = \mathrm{nnz}(\sextrac)$ for the number of edges in the coarsened graph, $E_{\mathrm{intra}} = \mathrm{nnz}(\sintra)$ for the number of intra-cluster edges, $b$ for the batch size in nodes and $n_{\mathrm{batch}}$ for the number of batches per epoch.

\input{tables/complexity.tex}

\paragraph{Classical GCN.} At each of the $K$ layers, the sparse matrix-vector product $\matrixpropag H^{(l-1)}$ costs $\mathcal{O}(Ed)$ and the linear transformation costs $\mathcal{O}(Nd^2)$. Storing all intermediate representations for backpropagation requires $\mathcal{O}(KNd)$ memory.

\paragraph{Coarsen-GNN.} Training operates entirely in the coarsened space of $n$ nodes and $E_c$ edges, with memory $\mathcal{O}(Knd)$. The lifting operation $QH_c^{(K)}$ at the final layer costs $\mathcal{O}(Nd)$ due to the sparsity of $Q$, adding an $\mathcal{O}(Nd)$ memory term for storing the lifted predictions.

\paragraph{Cluster-GCN.} Only intra-cluster edges are used, replacing $E$ by $E_{\mathrm{intra}} \leq E$ in the temporal complexity. The memory complexity is $\mathcal{O}(Kbd)$ per batch since only $b$ node representations are stored at any time.

\paragraph{CoRe-GNN.} Each layer performs both the inter-cluster propagation in the coarsened space at cost $\mathcal{O}(E_cd + nd^2)$ and the intra-cluster propagation at cost $\mathcal{O}(E_{\mathrm{intra}}d + Nd^2)$, plus two lifting and reduction operations at cost $\mathcal{O}(Nd)$ each due to the sparsity of $P$ and $Q$. Memory stores two sets of representations: $\mathcal{O}(KNd)$ for the original space and $\mathcal{O}(Knd)$ for the coarsened space.

\paragraph{CoRe-GNN batched.} The intra-cluster term is computed over the batch of $b$ nodes only, summing to $\mathcal{O}(K(E_{\mathrm{intra}}d + Nd^2))$ over all batches. The inter-cluster term is recomputed for each of the $n_{\mathrm{batch}}$ batches per epoch, contributing $\mathcal{O}(Kn_{\mathrm{batch}}(E_cd + nd^2))$, where the $d^2$ term accounts for the linear transformation in the coarsened space recomputed at each batch, giving a total time complexity of $\mathcal{O}(K(n_{\mathrm{batch}}E_cd + E_{\mathrm{intra}}d + (n_{\mathrm{batch}} + N)d^2))$. The key advantage over all compared architectures is that memory complexity is $\mathcal{O}(K(n + b)d)$, independent of $N$, making large-scale training tractable even when the full graph does not fit on GPU.

\subsection{Efficiency regimes for CoRe-GNN}
\label{app:efficiency_regimes}

CoRe-GNN has a lower theoretical time complexity than classical GCN when:
\begin{equation}
\label{eq:efficiency_condition}
E_c + E_{\mathrm{intra}} + 2N + nd \leq E
\end{equation}
Neither $E_c$ nor $E_{\mathrm{intra}}$ admits a closed-form expression as a function of $N$ and $r$ alone, as both depend on the graph structure and the coarsening algorithm. We analyze one tractable case below and then verify the condition empirically on a medium scale sparse graph ogbn-arxiv.

\paragraph{Dense graphs with balanced partitions.} To build intuition in a tractable theoretical setting, we analyze the efficiency condition under two simplifying assumptions: the original graph is dense with $E = N^2$, and the coarsening produces balanced clusters of equal size $N/n = 1/(1-r)$. Under these assumptions, each cluster forms a clique of size $N/n$, giving $E_{\mathrm{intra}} = n \cdot (N/n)^2 = N^2/n = N/(1-r)$, and in the worst case the coarsened graph is also complete with $E_c = n^2 = (1-r)^2N^2$. Condition~\eqref{eq:efficiency_condition} then becomes:
\[
(1-r)^2N^2 + \frac{N}{1-r} + 2N + (1-r)Nd \leq N^2
\]
For large $N$ with $d \ll N$, the dominant terms are $(1-r)^2N^2$ on the left and $N^2$ on the right, and the condition reduces to $(1-r)^2 \leq 1$, which holds for all $r \in [0,1]$. Dense graphs with balanced partitions therefore always satisfy the efficiency condition, regardless of the coarsening ratio, even in the worst case for $E_c$.

\paragraph{Empirical efficiency condition.} While the dense case is always favorable, graphs in practice are sparse. Figure~\ref{fig:complexity_curve} plots the left-hand side of condition~\eqref{eq:efficiency_condition} normalized by $E$, using the actual values of $E_c$ and $E_{\mathrm{intra}}$ computed from METIS coarsening on ogbn-arxiv ($N = 169{,}343$, $E = 1{,}166{,}243$), for $d \in \{64, 128, 256\}$. The condition is satisfied for sufficiently large $r$, but at very high coarsening ratios $E_{\mathrm{intra}}$ grows back as clusters absorb more edges internally and the left-hand side increases again toward the threshold. On this medium-scale sparse graph the gains are moderate, and more pronounced improvements are expected on larger or denser graphs. We also recall that beyond time complexity, CoRe-GNN admits a natural batching scheme that addresses memory scalability independently of this condition.

\begin{figure}[h!]
\centering
\includegraphics[width=0.95\linewidth]{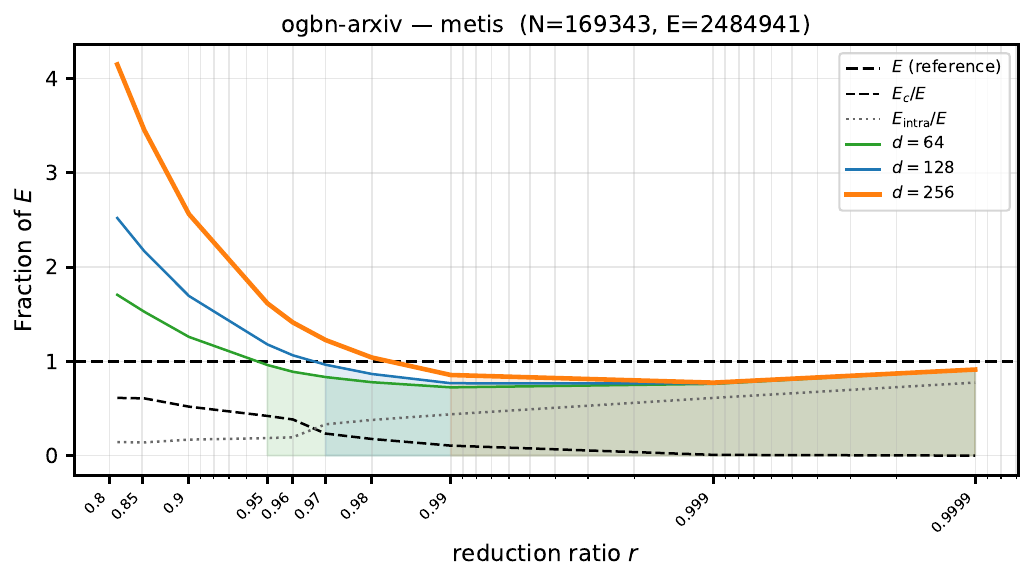}
\caption{Efficiency condition $E_c + E_{\mathrm{intra}} + 2N + nd \leq E$ (normalised by $E$) on ogbn-arxiv with METIS coarsening, for $d \in \{64, 128, 256\}$. Shaded area: regime where CoRe-GNN is computationally cheaper than classical GCN.}
\label{fig:complexity_curve}
\end{figure}

\section{Ablation Study of CoRe-GNN}
\label{app:ablation}

We study four design choices of CoRe-GNN: the normalization of the lifting matrix $Q$ (Table~\ref{tab:abl_qnorm}), the choice of inter-cluster propagation matrix (Table~\ref{tab:abl_psq}), weight sharing between the two branches (Table~\ref{tab:abl_sw}), and layer normalization (Table~\ref{tab:abl_ln}). All ablations are run with the same hyperparameters as the best results presented in the main manuscript, varying only the factor under study.

\paragraph{Normalization of the lifting matrix $Q$.} The lifting matrix $Q$ can be defined in two ways. The combinatorial version $\mathcal{Q}$ has binary entries, with exactly one non-zero entry of value $1$ per row. The normalized version uses $Q = (D+I)^{-1/2} \mathcal{Q} (D_c+I_c)^{1/2}$, where $D$ and $A$ are the degree matrix and adjacency matrix of $G$, and $D_c$, $A_c$ their counterparts on $G_c$. This normalization is associated with the self-loop normalized Laplacian $L = (D+I)^{-1/2}(D-A)(D+I)^{-1/2}$, which satisfies $L = I_N - S$. In both cases, $P = Q^+$ is the Moore-Penrose pseudoinverse of the chosen $Q$, and both verify $L_c = Q^\top L Q$ and $\mathcal{L}_c = D_c - A_c = \mathcal{Q}^\top \mathcal{L}\mathcal{Q} = \mathcal{Q}^\top (D-A)\mathcal{Q} $. Table~\ref{tab:abl_qnorm} shows that the normalized variant performs better on most datasets, with a mean accuracy gain of $+1.51\%$, and is used as default throughout the paper.

\begin{table}[!h]
\centering
\caption{Impact of lifting matrix normalization: normalized self-loop (default) vs combinatorial}
\scalebox{1}{
\begin{tabular}{lcc}
\toprule
Dataset & Q=combinatorial & Q=norm (default) \\
\midrule
Cora & $81.48 \pm 0.39$ & $81.67 \pm 0.64$ \\
CiteSeer & $68.28 \pm 1.46$ & $73.62 \pm 0.49$ \\
PubMed & $79.56 \pm 0.31$ & $79.27 \pm 0.35$ \\
chameleon & $59.87 \pm 3.00$ & $60.68 \pm 1.99$ \\
squirrel & $46.08 \pm 1.00$ & $46.80 \pm 1.13$ \\
Amazon-ratings & $47.44 \pm 0.45$ & $47.45 \pm 0.61$ \\
paris & $39.07 \pm 3.21$ & $42.86 \pm 1.67$ \\
\bottomrule
\end{tabular}
}
\label{tab:abl_qnorm}
\end{table}

\paragraph{Choice of inter-cluster propagation matrix.} CoRe-GNN uses $\sextrac = P\sextra Q$ as the inter-cluster propagation matrix, where $\sextra = \matrixpropag - \sintra$ contains only inter-cluster edges. An alternative is to use the full coarsened propagation matrix $S_c = PSQ$, which additionally includes the intra-cluster term $P\sintra Q$. Table~\ref{tab:abl_psq} shows that $\sextrac$ performs better on most datasets, with a mean accuracy gain of $+2.62\%$. This is consistent with avoiding double-counting the intra-cluster contribution, which is already handled by the $\sintra$ branch.

\begin{table}[!h]
\centering
\caption{Impact of propagation matrix:  $\sextrac = P\sextra Q$ (default) vs full $\newmatrixpropag = PSQ $. }
\scalebox{1}{
\begin{tabular}{lcc}
\toprule
Dataset & $\newmatrixpropag = PSQ $ & $\sextrac = P\sextra Q$ (default) \\
\midrule
Cora & $80.87 \pm 0.74$ & $81.67 \pm 0.64$ \\
CiteSeer & $68.58 \pm 2.52$ & $73.62 \pm 0.49$ \\
PubMed & $78.65 \pm 0.17$ & $79.27 \pm 0.35$ \\
chameleon & $60.94 \pm 2.77$ & $60.68 \pm 1.99$ \\
squirrel & $43.54 \pm 1.59$ & $46.80 \pm 1.13$ \\
Amazon-ratings & $47.44 \pm 0.45$ & $47.45 \pm 0.61$ \\
paris & $33.99 \pm 4.62$ & $42.86 \pm 1.67$ \\
\bottomrule
\end{tabular}
}
\label{tab:abl_psq}
\end{table}

\paragraph{Weight sharing.} By default, CoRe-GNN uses independent weight matrices $\theta^{(l)}$ and $\theta_c^{(l)}$ for the intra-cluster and inter-cluster branches respectively. A simpler variant ties the two weights, setting $\theta^{(l)} = \theta_c^{(l)}$. Table~\ref{tab:abl_sw} shows that independent weights perform better on most datasets, with a mean accuracy gain of $+2.30\%$. This suggests that the two branches benefit from learning distinct transformations, as the intra-cluster and inter-cluster propagations operate on different graph structures and at different scales.

\begin{table}[!h]
\centering
\caption{Impact of weight sharing (SW) between GCN branches}
\scalebox{1}{
\begin{tabular}{lcc}
\toprule
Dataset & SW=True &SW=False (default)  \\
\midrule
Cora &  $81.57 \pm 0.61$ & $81.67 \pm 0.64$  \\
CiteSeer  & $70.94 \pm 1.07$ & $ 73.62 \pm 0.49$  \\
PubMed &  $80.41 \pm 0.23$ & $79.27 \pm 0.35$  \\
chameleon & $58.16 \pm 2.60$ & $60.68 \pm 1.99$  \\
squirrel & $42.92 \pm 2.07$ & $46.80 \pm 1.13$  \\
Amazon-ratings & $47.28 \pm 0.46$ & $47.45 \pm 0.61$  \\
paris & $34.94 \pm 1.90$ & $42.86 \pm 1.67$  \\
\bottomrule
\end{tabular}
}
\label{tab:abl_sw}
\end{table}

\paragraph{Layer normalization.} Layer normalization normalizes representations across the feature dimension, independently of the number of nodes. This makes it compatible with our architecture where cluster sizes vary across batches, unlike batch normalization which depends on the number of nodes and would produce inconsistent statistics across batches of different sizes. Table~\ref{tab:abl_ln} shows that the best choice is dataset-dependent: layer normalization is beneficial on Amazon Ratings and Paris, while the absence of normalization performs better on the remaining datasets. We therefore select the normalization setting per dataset during hyperparameter tuning.

\begin{table}[!h]
\centering
\caption{Impact of layer normalization. Bold = chosen setting.}
\scalebox{1}{
\begin{tabular}{lcc}
\toprule
Dataset & LN True & LN False \\
\midrule
Cora & $76.07 \pm 2.33$ & $\mathbf{81.67 \pm 0.64}$ \\
CiteSeer & $69.46 \pm 2.75$ & $\mathbf{73.62 \pm 0.49}$ \\
PubMed & $75.36 \pm 1.48$ & $\mathbf{79.27 \pm 0.35}$ \\
chameleon & $59.41 \pm 2.41$ & $\mathbf{60.68 \pm 1.99}$ \\
squirrel & $40.64 \pm 1.30$ & $\mathbf{46.80 \pm 1.13}$ \\
Amazon-ratings & $\mathbf{47.45 \pm 0.61}$ & $41.66 \pm 0.39$ \\
paris & $\mathbf{42.86 \pm 1.67}$ & $38.59 \pm 1.32$ \\
\bottomrule
\end{tabular}
}
\label{tab:abl_ln}
\end{table}

\section{Pseudo code of Proposed methods}
\label{app:pseudocode}
Algorithm~\ref{alg:coarsen-training} describes the training forward pass of Coarsen-GNN, which operates entirely in the coarsened space and maps the final predictions back to the original nodes.

\begin{algorithm}[H]
\small
\caption{Coarsen-GNN: Training Forward Pass}\label{alg:coarsen-training}
\begin{algorithmic}[1]
\Require $H^{(0)} \in \mathbb{R}^{N \times d}$ \textit{(original node features)},\;
         $\newmatrixpropag = PSQ$ \textit{(precomputed)},\;
         $P, Q$,\;
         weights $(\theta_c^{(l)})_{l=1}^K$,\;
         train mask on original graph
\State $H_c^{(0)} \leftarrow P\, H^{(0)}$
         \hfill\textit{// reduce to coarsened space}
\For{$l = 1, \ldots, K$}
  \State $H_c^{(l)} \leftarrow \sigma\left(\newmatrixpropag\, H_c^{(l-1)}\theta_c^{(l)}\right)$
         \hfill\textit{// propagation in coarsened space}
\EndFor
\State $H^{(K)} \leftarrow Q\, H_c^{(K)}$
         \hfill\textit{// lift predictions to original space}
\State compute loss on $H^{(K)}[\text{train\_mask}]$
\State backpropagate and update $(\theta_c^{(l)})$
\end{algorithmic}
\end{algorithm}

Algorithm~\ref{alg:coregnn-full} describes the full batch training forward pass of CoRe-GNN, maintaining two parallel states $H_c^{(l)}$ and $H^{(l)}$ at each layer.

\begin{algorithm}[H]
\small
\caption{CoRe-GNN: Full Batch Training Forward Pass}\label{alg:coregnn-full}
\begin{algorithmic}[1]
\Require $H^{(0)} \in \mathbb{R}^{N \times d}$ \textit{(original node features)},\;
         $P, Q, \sextrac$ \textit{(precomputed)},\;
         $\sintra$,\;
         weights $(\theta^{(l)}, \theta_c^{(l)})_{l=1}^K$,\;
         train mask on original graph
\State $H_c^{(0)} \leftarrow P\, H^{(0)}$
         \hfill\textit{// reduce to coarsened space}
\For{$l = 1, \ldots, K$}
  \State $H_c^{(l)} \leftarrow \sigma\left(\sextrac\, H_c^{(l-1)}\theta_c^{(l)}\right)$
         \hfill\textit{// inter-cluster propagation, coarsened space}
  \State $H^{(l)} \leftarrow \sigma\left(\sintra\, H^{(l-1)}\theta^{(l)}\right) + Q\, H_c^{(l)}$
         \hfill\textit{// intra-cluster + lift inter-cluster}
  \State $H_c^{(l)} \leftarrow P\, H^{(l)}$
         \hfill\textit{// reduce to coarsened space for next layer}
\EndFor
\State compute loss on $H^{(K)}[\text{train\_mask}]$
\State backpropagate and update $(\theta^{(l)}, \theta_c^{(l)})$
\end{algorithmic}
\end{algorithm}

Algorithm~\ref{alg:coregnn-batched-extended} describes the batched training forward pass, where only the clusters selected in the current batch perform full intra-cluster propagation while the remaining super-nodes receive a self-loop proxy.

\begin{algorithm}[h!]
\small
\caption{CoRe-GNN: Batched Training Forward Pass (extended)}\label{alg:coregnn-batched-extended}
\begin{algorithmic}[1]
\Require $H^{(0)} \in \mathbb{R}^{N \times d}$ \textit{(original node features)},\;
         $P_\mathcal{B}, Q_\mathcal{B}$ \textit{(restricted to batch)},\;
         $\sextrac, \sintra[\mathcal{B}]$ \textit{(precomputed)},\;
         batch $\mathcal{B} \subseteq [n]$,\;
         weights $(\theta^{(l)}, \theta_c^{(l)})_{l=1}^K$,\;
         train mask on original graph
\State $H_c^{(0)} \leftarrow P\, H^{(0)}$
         \hfill\textit{// reduce full graph to coarsened space}
\State $H_{\mathcal{B}}^{(0)} \leftarrow H^{(0)}[\mathcal{B}]$
         \hfill\textit{// extract batch nodes}
\For{$l = 1, \ldots, K$}
  \State $H_c^{(l)} \leftarrow \sigma\left(\sextrac\, H_c^{(l-1)}\theta_c^{(l)}\right)$
         \hfill\textit{// all $n$ super-nodes, coarsened space}
  \State $H_{\mathcal{B}}^{(l)} \leftarrow \sigma\left(\sintra[\mathcal{B}]\, H_{\mathcal{B}}^{(l-1)}\theta^{(l)}\right)$
         \hfill\textit{// intra-cluster, batch nodes only}
  \State $H_{\mathcal{B}}^{(l)} \leftarrow H_{\mathcal{B}}^{(l)} + Q_\mathcal{B}\, H_c^{(l)}[\mathcal{B}]$
         \hfill\textit{// lift inter-cluster to batch nodes}
  \State $H_c^{(l)}[\mathcal{B}] \leftarrow P_\mathcal{B}\, H_{\mathcal{B}}^{(l)}$
         \hfill\textit{// reduce batch nodes back to super-nodes}
  \State $H_c^{(l)}[\bar{\mathcal{B}}] \leftarrow \sigma\left(H_c^{(l-1)}[\bar{\mathcal{B}}]\,\theta^{(l)}\right) + H_c^{(l)}[\bar{\mathcal{B}}]$
         \hfill\textit{// self-loop proxy for non-batch super-nodes}
\EndFor
\State $\text{train\_mask}_\mathcal{B} \leftarrow \text{train\_mask}[\text{index\_in\_original\_batch}]$
         \hfill\textit{// restrict mask to current batch}
\State compute loss on $H_{\mathcal{B}}^{(K)}[\text{train\_mask}_\mathcal{B}]$
\State backpropagate and update $(\theta^{(l)}, \theta_c^{(l)})$
\end{algorithmic}
\end{algorithm}

\subsection{Pseudo-code of exact inference}
\label{app:inference}
Unlike training, exact inference introduces no approximation: every cluster is processed exactly, layer by layer. At each layer, $\sintra[\mathcal{B}]$ is loaded batch by batch from CPU to GPU on demand, while $H^{(l)}$ is written back to CPU after each layer to free GPU memory. This strategy is not applicable during training since backpropagation requires retaining the full computational graph across all layers, which is precisely why the self-loop proxy of Algorithm~\ref{alg:core-batched} is needed at training time.

\begin{algorithm}[h!]
\small
\caption{CoRe-GNN: Exact Layer-by-Layer Inference}\label{alg:core-inference}
\begin{algorithmic}[1]
\Require $H^{(0)} \in \mathbb{R}^{N \times d}$ \textit{(on GPU)},\;
         $P, Q, \sextrac$ \textit{(on GPU)},\;
         $\sintra[\mathcal{B}]$ \textit{(pre-split by cluster, on CPU)},\;
         weights $(\theta^{(l)}, \theta_c^{(l)})_{l=1}^K$
\For{$l = 1, \ldots, K$}
  \State $H_c^{(l)} \leftarrow \sigma\left(\sextrac\, P H^{(l-1)} \theta_c^{(l)}\right)$
         \hfill\textit{// inter-cluster, fully on GPU}
  \State initialize $H^{(l)} \leftarrow Q H_c^{(l)}$
         \hfill\textit{// lift to original space, on GPU}
  \For{each cluster batch $\mathcal{B}$}
    \State load $\sintra[\mathcal{B}]$ from CPU to GPU
    \State $H^{(l)}[\mathcal{B}] \leftarrow H^{(l)}[\mathcal{B}] + \sigma\left(\sintra[\mathcal{B}]\, H^{(l-1)}[\mathcal{B}]\,\theta^{(l)}\right)$
           \hfill\textit{// intra-cluster batch, on GPU}
    \State discard $\sintra[\mathcal{B}]$ from GPU
  \EndFor
  \State discard $H^{(l-1)}$
\EndFor
\State \textbf{return} $H^{(K)}$
\end{algorithmic}
\end{algorithm}

\section{Datasets}
\label{app:datasets}

We restrict Cora and Citeseer to their principal connected component (PCC) as coarsening algorithms tend to process small connected components first, leading to poor results at low coarsening ratios. This makes direct comparison with other benchmarks more difficult as the training and evaluation sets differ. The characteristics of these datasets and their PCC are reported in Table~\ref{tab:caracCora}. Details for all datasets are in Table~\ref{tab:dataset_stats}.

\paragraph{Homophilic datasets.} Cora, Citeseer and Pubmed~\cite{yang2016revisiting} are citation networks where nodes are papers and edges are citations. Labels correspond to research topics and are strongly homophilic.

\paragraph{Heterophilic datasets.} Chameleon and Squirrel~\cite{rozemberczki2021multi} are Wikipedia page networks. Amazon Ratings~\cite{platonovcritical} is a co-purchase network where nodes are products. These datasets exhibit low homophily, making them challenging for local message passing methods.

\paragraph{Long-range datasets.} Paris and London~\cite{liang2025towards} are road networks specifically designed to measure long-range interactions in graph machine learning.

\paragraph{Large-scale datasets.} ogbn-arxiv, Reddit~\cite{hamilton2017inductive} and ogbn-products~\cite{hu2020open} are large graphs with hundreds of thousands to millions of nodes, used to evaluate scalability.

\begin{table}[ht]
\centering
\caption{Characteristics of Cora and CiteSeer Datasets and their principal connected components}
\vspace{10pt}
\begin{tabular}{lccccc}
\toprule
\textbf{Dataset} & \textbf{\# Nodes} & \textbf{\# Edges}  & \textbf{\# Train Nodes} & \textbf{\# Val Nodes} & \textbf{\# Test Nodes} \\
\midrule
Cora & 2,708 & 10,556 & 140 & 500 & 1,000   \\
Cora PCC & 2,485 & 10,138  & 122 & 459 & 915 \\
\midrule
Citeseer & 3,327 & 9,104 & 120 & 500 & 1,000   \\
Citeseer PCC & 2,120 & 7,358   & 80 & 328 & 663 \\

\bottomrule
\label{tab:caracCora}
\end{tabular}
\end{table}

\input{tables/dataset_stats.tex}

\section{Coarsening Algorithms}
\label{app:coarsening_ablation}
\label{app:coarsening_algos}

We briefly describe the three coarsening algorithms used in this work. Among them, only METIS enforces \emph{balanced} partitions by construction, while Graclus and Loukas do not impose cluster size constraints. Due to its higher computational cost, Loukas does not scale to very large graphs and was therefore not evaluated on ogbn-products.

\paragraph{METIS.}~\cite{karypis1997metis} partitions the graph into $n$ balanced clusters minimizing the normalized cut. Since METIS takes a target number of nodes as input rather than a ratio, we compute $n = \lfloor(1-r)N\rfloor$ from the desired coarsening ratio $r$.

\paragraph{Graclus.}~\cite{dhillon2007weighted} applies a multilevel clustering based on the ratio cut. Since Graclus does not directly accept a target coarsening ratio, we apply it iteratively until $n \leq (1-r)N$. The achieved ratio is therefore approximate, which makes this algorithm better suited to large graphs where small deviations in the number of super-nodes are acceptable.

\subsection{Adaptation of Loukas Algorithm}
\label{app:loukasalg}

The Loukas algorithm~\cite{loukas2019graph} constructs a coarsened graph by minimizing the RSA constant, ensuring that smooth signals on the original graph are well preserved after coarsening.
Algorithm~\ref{algo:loukas_adapted} presents our adaptation of the Loukas coarsening algorithm. The algorithm iteratively selects contraction sets, merges the corresponding nodes, and updates the graph structure. Our main modification replaces the combinatorial Laplacian by any $\Delta$-Laplacian $L = \Delta \mathcal{L} \Delta$~\cite{joly2025taxonomy}, and removes the diagonal of $A_c$ at each iteration, which empirically yields lower RSA values. 

\begin{algorithm}[H]
\small
\caption{Loukas Algorithm (adapted)}\label{algo:loukas_adapted}
\begin{algorithmic}[1]
\Require Adjacency matrix $A$, Laplacian $L$, coarsening ratio $r$, preserved space $\mathcal{R}$, max nodes per step $n_e$
\State $n_{\mathrm{obj}} \leftarrow \lfloor N(1-r) \rfloor$
\State compute $B_0 \leftarrow VV^\top L^{-1/2}$ with $V$ an orthonormal basis of $\mathcal{R}$
\State $Q \leftarrow I_N$
\While{$n \geq n_{\mathrm{obj}}$}
    \State generate candidate contraction sets
    \For{each contraction set $\mathcal{C}$}
        \State compute $\mathrm{cost}(\mathcal{C}, B_{l-1}, L_{l-1}) = \frac{\| \Pi_\mathcal{C} B_{l-1}(B_{l-1}^\top L_{l-1} B_{l-1})^{-1/2} \|_{L_\mathcal{C}}}{|\mathcal{C}|-1}$
    \EndFor
    \State sort contraction sets by increasing cost
    \State select non-overlapping sets until $\min(n - n_{\mathrm{obj}}, n_e)$ nodes are merged
    \State compute binary lifting matrix $\mathcal{Q}_l$ from selected contraction sets
    \State $Q_l  \leftarrow  \text{Normalize $\mathcal{Q}_l$ according to the laplacian (i.e if $L = \Delta L \Delta$ then $Q = \Delta_{l-1}^{-1} \mathcal{Q}_l \Delta_l $}  $)
    \State $Q \leftarrow Q \cdot Q_l $
    \State $P_l \leftarrow \text{Moore-Penrose}(\mathcal{Q}_l) $
    \State $B_l \leftarrow P_l B_{l-1}$
    \State $A_l \leftarrow \mathcal{Q}_l^\top A_{l-1} \mathcal{Q}_l - \mathrm{diag}(\mathcal{Q}_l^\top A_{l-1} \mathcal{Q}_l \mathbf{1}_n)$
    \State $L_l \leftarrow Q_l^\top L_{l-1} Q_l$
\EndWhile
\State $P \leftarrow (Q^\top Q)^{-1} Q^\top$
\Return $A_c,\, Q,\, P$
\end{algorithmic}
\end{algorithm}

\paragraph{Candidate contraction sets.} Contraction sets can be defined either as pairs of adjacent nodes (edge-based) or as full node neighborhoods. As neighborhoods tend to be large in our graphs, the neighborhood variant proves impractical at low coarsening ratios. We therefore use edge-based contraction sets and control the greedy behavior via $n_e$.

\paragraph{Coarsening Hyperparameters of this paper}
By default in this paper, $L$ is the self loop normalized laplacian $L = (D+I)^{-1/2}(D-A)(D+I)^{-1/2} $ such that $S= I_N - L$. The preserved space is the first $K$ eigenvectors of $L$ with $K= 100$. $n_e$ the number of edges contracted at each step is defined as $10\%$ of the original graph.

\section{Hyperparameters}
\label{app:hyperparameters}

\subsection{Presentation of hyperparameters}

All experiments use the Adam optimizer with a ReduceLROnPlateau scheduler, trained for up to 800 epochs with dropout fixed at $0.5$ on a single NVIDIA A40 GPU. The hyperparameter grid covers the coarsening algorithm (METIS, Graclus, Loukas), coarsening ratio $r$, hidden dimension, network depth $K$, learning rate, weight decay, and layer normalization. Due to computational constraints, the full grid search was run over 3 seeds, while the final results reported in the main paper are averaged over 10 seeds.

\subsection{Chosen hyperparameters for main results}

We report below the hyperparameters selected per dataset and per model for the results presented in the main paper.

\input{tables/best_hp_for_table}

\subsection{Hyperparameter grids}
Table~\ref{tab:hyperparameter_grid} reports the full hyperparameter grid explored during the search phase.
\input{tables/hyperparameter_grid.tex}

\section{Graph structure underlying Figure~\ref{fig:prop_matrices_v2}}Figure~\ref{fig:graph_bilevel_5cluster} shows the original graph $G$ and its coarsened counterpart $G_c$ corresponding to the partition illustrated in Figure~\ref{fig:prop_matrices_v2}.

\input{tikz_figure/graph_bilevel_5cluster.tikz}

\end{document}

%% file: tikz_figure/coregnn_layer_flow.tikz
\begin{figure}[t]
\centering
\resizebox{\linewidth}{!}{%
\begin{tikzpicture}[
    every node/.style={inner sep=0pt},
    gnode/.style={circle, fill=#1, draw=#1!70!black, minimum size=5pt},
    gnode/.default={black},
    supernode/.style={circle, fill=#1, draw=#1!70!black,
                      minimum size=9pt, inner sep=0pt},
    supernode/.default={black},
    box/.style={draw=gray!60, rounded corners=4pt, thick, align=center,
                fill=white, font=\small, inner sep=6pt},
    arr/.style={-stealth, thick, gray!75},
    plusnode/.style={circle, draw=black!70, thick,
                     minimum size=16pt, font=\large, fill=white},
    signode/.style={circle, draw=black!60, semithick,
                    minimum size=14pt, font=\small, fill=white},
    lblfont/.style={font=\small, text=gray!65!black},
]

\definecolor{c0}{RGB}{80,130,220}
\definecolor{c1}{RGB}{210,80,80}
\definecolor{c2}{RGB}{60,160,80}

\def\ymid{0.1}    
\def\ytop{1.3}    
\def\ybot{-1.1}   

\node[gnode=c0] (L0) at (0,   1.1) {};
\node[gnode=c0] (L1) at (0.6, 0.6) {};
\node[gnode=c0] (L2) at (-0.1,0.2) {};
\draw[c0,thick] (L0)--(L1) (L1)--(L2);

\node[gnode=c1] (L3) at (-0.4,-0.3) {};
\node[gnode=c1] (L4) at (0.2, -0.2) {};
\node[gnode=c1] (L5) at (-0.1,-0.9) {};
\node[gnode=c1] (L6) at (0.6, -0.7) {};
\draw[c1,thick] (L3)--(L4) (L4)--(L6) (L3)--(L5) (L5)--(L6);

\node[gnode=c2] (L7)  at (1.1, 0.8) {};
\node[gnode=c2] (L8)  at (1.6, 0.3) {};
\node[gnode=c2] (L9)  at (1.4,-0.4) {};
\node[gnode=c2] (L10) at (0.9,-0.1) {};
\node[gnode=c2] (L11) at (1.7,-0.8) {};
\draw[c2,thick] (L7)--(L8) (L8)--(L9) (L9)--(L10) (L7)--(L10) (L9)--(L11);

\draw[gray,thin] (L1)--(L4) (L2)--(L3) (L1)--(L10)
                        (L6)--(L9) (L7)--(L0); 

\node[font=\normalsize\bfseries] at (0.65,-1.25) {$H^{(l)}$};


\node[plusnode] (plus)     at (10.5, \ymid)  {$+$};
\node[box] (topbox)        at (5.7,  \ytop)
  {$S^{\mathrm{intra}}\, H^{(l)}\, \theta^{(l)}$};
\node[signode] (sigtop)    at (7.95,  \ytop)  {$\sigma$};
\node[signode] (sigbot)    at (7.95,  \ybot)  {$\sigma$};

\node[box] (reducebox)     at (3.8,  \ybot)  {Reduce\\[2pt]$P$};
\node[box] (botbox)        at (6.1,  \ybot)
  {$S_{c}^{\mathrm{inter}}H^{(l)}_c\theta_c^{(l)}$};
\node[box] (liftbox)       at (9.0,  \ybot)  {Lift\\[2pt]$Q$};

\draw[arr] (1.85, \ymid) to[out=40,  in=180] (3.0, \ytop) -- (topbox.west);
\draw[arr] (1.85, \ymid) to[out=-40, in=180] (3.0, \ybot) -- (reducebox.west);

\draw[arr] (topbox.east) -- (sigtop.west);
\draw[arr] (sigtop.east) -- ++(1.0,0) to[out=0, in=90] (plus.north);

\begin{scope}[opacity=0.3]
\node[gnode=c0] (T0) at (4.3, 2.4) {};
\node[gnode=c0] (T1) at (4.8, 2.1) {};
\node[gnode=c0] (T2) at (4.3, 1.95) {};
\draw[c0,semithick] (T0)--(T1) (T1)--(T2);
\end{scope}

\node[gnode=c1] (T3) at (5.41, 2.27) {};
\node[gnode=c1] (T4) at (5.74, 2.33) {};
\node[gnode=c1] (T5) at (5.57, 1.94) {};
\node[gnode=c1] (T6) at (5.96, 2.05) {};
\draw[c1,semithick] (T3)--(T4) (T4)--(T6) (T3)--(T5) (T5)--(T6);

\begin{scope}[opacity=0.3]
\node[gnode=c2] (T7)  at (6.62, 2.45) {};
\node[gnode=c2] (T8)  at (6.82, 2.29) {};
\node[gnode=c2] (T9)  at (6.74, 2.08) {};
\node[gnode=c2] (T10) at (6.54, 2.17) {};
\node[gnode=c2] (T11) at (7.02, 1.88) {};
\draw[c2,semithick]
  (T7)--(T8) (T8)--(T9) (T9)--(T10) (T7)--(T10) (T9)--(T11);
\end{scope}

\draw[arr] (reducebox.east) -- (botbox.west);
\draw[arr] (botbox.east)    -- (sigbot.west);
\draw[arr] (sigbot.east)    -- (liftbox.west);
\draw[arr] (liftbox.east) to[out=0, in=-90] (plus.south);

\node[supernode=c0] (S0) at (5.50, +0.08) {};
\node[supernode=c1] (S1) at (5.85, -0.42) {};
\node[supernode=c2] (S2) at (6.70, -0.18) {};
\draw[gray!65, line width=1.4pt] (S0)--(S1) (S1)--(S2) (S0)--(S2);


\draw[arr] (plus.east) -- (11.5, \ymid);

\node[gnode=c0] (R0) at (11.7, 1.1) {};
\node[gnode=c0] (R1) at (12.3, 0.6) {};
\node[gnode=c0] (R2) at (11.6, 0.2) {};
\draw[c0,thick] (R0)--(R1) (R1)--(R2);

\node[gnode=c1] (R3) at (11.3,-0.3) {};
\node[gnode=c1] (R4) at (11.9,-0.2) {};
\node[gnode=c1] (R5) at (11.6,-0.9) {};
\node[gnode=c1] (R6) at (12.3,-0.7) {};
\draw[c1,thick] (R3)--(R4) (R4)--(R6) (R3)--(R5) (R5)--(R6);

\node[gnode=c2] (R7)  at (12.8, 0.8) {};
\node[gnode=c2] (R8)  at (13.3, 0.3) {};
\node[gnode=c2] (R9)  at (13.1,-0.4) {};
\node[gnode=c2] (R10) at (12.6,-0.1) {};
\node[gnode=c2] (R11) at (13.4,-0.8) {};
\draw[c2,thick] (R7)--(R8) (R8)--(R9) (R9)--(R10) (R7)--(R10) (R9)--(R11);

\draw[gray,thin]
  (R1)--(R4) (R2)--(R3) (R1)--(R10) (R6)--(R9) (R7)--(R0); 

\node[font=\normalsize\bfseries] at (12.35,-1.25) {$H^{(l+1)}$};

\end{tikzpicture}%
}

\caption{One CoRe-GNN batched layer. $H^{(l)}$ flows through two parallel branches. Top: intra-cluster message passing, where only the active cluster (red) is processed per step while others are handled in separate batches (in the full-batch setting, all clusters are processed simultaneously). Bottom: inter-cluster propagation on the coarsened graph over all $n$ super-nodes at once, reduce ($P$), propagate ($\sextrac$), activate ($\sigma$), lift back ($Q$). Both outputs are summed to produce $H^{(l+1)}$.}
\label{fig:core_gnn_layer}
\end{figure}

%% file: tikz_figure/prop_matrices_5cluster.tikz
\begin{figure}[ht]

\centering
\resizebox{\linewidth}{!}{%
\begin{tikzpicture}[x=0.18cm, y=-0.18cm]

  \colorlet{KC0}{blue!60}
  \colorlet{KC1}{red!55}
  \colorlet{KC2}{green!65!black}
  \colorlet{KC3}{orange!80}
  \colorlet{KC4}{violet!55}
  \colorlet{Kfull}{cyan!15!blue!30}

  \def\Ns{20}
  \pgfmathsetmacro{\DX}{(\Ns+4)*0.18}  

  \def\hrfix{0.5}

  \newcommand{\clusterbordersV}{%
    \foreach \k in {0,4,8,12,15,20}{%
      \draw[black!50, line width=0.6pt]
        (\k,0)--(\k,\Ns) (0,\k)--(\Ns,\k);%
    }%
    \draw[black, line width=0.8pt] (0,0) rectangle (\Ns,\Ns);%
  }

  \begin{scope}[xshift=0cm]
    \fill[white] (0,0) rectangle (\Ns,\Ns);
    \foreach \c/\r in {
      0/0,1/1,2/2,3/3,4/4,5/5,6/6,7/7,8/8,9/9,
      10/10,11/11,12/12,13/13,14/14,15/15,16/16,17/17,18/18,19/19,
      1/0,0/1,  2/1,1/2,  3/0,0/3,
      5/4,4/5,  6/5,5/6,  7/4,4/7,
      9/8,8/9,  10/9,9/10,  11/8,8/11,
      13/12,12/13,  14/13,13/14,
      16/15,15/16,  17/16,16/17,  18/17,17/18,  19/15,15/19,
      5/1,1/5,  4/3,3/4,  6/2,2/6,
      9/0,0/9,  8/2,2/8,  10/3,3/10,
      12/6,6/12,  14/7,7/14,  13/5,5/13,
      15/10,10/15,  17/11,11/17,  19/9,9/19,  16/8,8/16,
      16/12,12/16,  18/14,14/18,  15/13,13/15%
    }{
      \fill[gray!40] (\c,\r) rectangle (\c+1,\r+1);
    }
    \clusterbordersV
    \node[above=4pt, font=\small\bfseries] at (10,0) {GCN};
    \node[below=5pt, font=\small]          at (10,\Ns) {$S$};
  \end{scope}

  \begin{scope}[xshift=\DX cm]
    \fill[white] (0,0) rectangle (\Ns,\Ns);
    \foreach \c/\r in {
      0/0,1/1,2/2,3/3, 1/0,0/1, 2/1,1/2, 3/0,0/3}{
      \fill[KC0] (\c,\r) rectangle (\c+1,\r+1);}
    \foreach \c/\r in {
      4/4,5/5,6/6,7/7, 5/4,4/5, 6/5,5/6, 7/4,4/7}{
      \fill[KC1] (\c,\r) rectangle (\c+1,\r+1);}
    \foreach \c/\r in {
      8/8,9/9,10/10,11/11, 9/8,8/9, 10/9,9/10, 11/8,8/11}{
      \fill[KC2] (\c,\r) rectangle (\c+1,\r+1);}
    \foreach \c/\r in {
      12/12,13/13,14/14, 13/12,12/13, 14/13,13/14}{
      \fill[KC3] (\c,\r) rectangle (\c+1,\r+1);}
    \foreach \c/\r in {
      15/15,16/16,17/17,18/18,19/19,
      16/15,15/16, 17/16,16/17, 18/17,17/18, 19/15,15/19}{
      \fill[KC4] (\c,\r) rectangle (\c+1,\r+1);}
    \clusterbordersV
    \node[above=4pt, font=\small\bfseries] at (10,0) {Cluster-GCN};
    \node[below=5pt, font=\small]          at (10,\Ns) {$S^{\mathrm{intra}}$};
  \end{scope}

  \pgfmathsetmacro{\XD}{2*\DX}
  \begin{scope}[xshift=\XD cm]
    \fill[white] (0,0) rectangle (\Ns,\Ns);
    \foreach \ci/\cs/\csz/\col in {
      0/0/4/KC0, 1/4/4/KC1, 2/8/4/KC2, 3/12/3/KC3, 4/15/5/KC4}{
      \pgfmathsetmacro{\cx}{\cs+\csz/2}
      \fill[\col] (\cx-\hrfix,\cx-\hrfix) rectangle (\cx+\hrfix,\cx+\hrfix);
      \draw[black, line width=0.6pt]
        (\cx-\hrfix,\cx-\hrfix) rectangle (\cx+\hrfix,\cx+\hrfix);
    }
    \foreach \csi/\cszi/\csj/\cszj/\ca/\cb in {
      0/4/4/4/KC0/KC1,   4/4/0/4/KC1/KC0,
      0/4/8/4/KC0/KC2,   8/4/0/4/KC2/KC0,
      4/4/12/3/KC1/KC3,  12/3/4/4/KC3/KC1,
      8/4/15/5/KC2/KC4,  15/5/8/4/KC4/KC2,
      12/3/15/5/KC3/KC4, 15/5/12/3/KC4/KC3%
    }{
      \pgfmathsetmacro{\cx}{\csi+\cszi/2}
      \pgfmathsetmacro{\cy}{\csj+\cszj/2}
      \fill[\ca] (\cx-\hrfix,\cy-\hrfix) rectangle (\cx,         \cy+\hrfix);
      \fill[\cb] (\cx,        \cy-\hrfix) rectangle (\cx+\hrfix,  \cy+\hrfix);
      \draw[black, line width=0.6pt]
        (\cx-\hrfix,\cy-\hrfix) rectangle (\cx+\hrfix,\cy+\hrfix);
    }
    \clusterbordersV
    \node[above=4pt, font=\small\bfseries] at (10,0) {Coarsen-GNN};
    \node[below=5pt, font=\small]          at (10,\Ns) {$S_c$};
  \end{scope}

  \pgfmathsetmacro{\XF}{3*\DX}
  \begin{scope}[xshift=\XF cm]
    \fill[white] (0,0) rectangle (\Ns,\Ns);
    \pgfmathsetmacro{\cxA}{0+4/2}
    \fill[KC0] (\cxA-\hrfix,\cxA-\hrfix) rectangle (\cxA+\hrfix,\cxA+\hrfix);
    \foreach \c/\r in {
      4/4,5/5,6/6,7/7, 5/4,4/5, 6/5,5/6, 7/4,4/7}{
      \fill[KC1] (\c,\r) rectangle (\c+1,\r+1);}
    \pgfmathsetmacro{\cxC}{8+4/2}
    \fill[KC2] (\cxC-\hrfix,\cxC-\hrfix) rectangle (\cxC+\hrfix,\cxC+\hrfix);
    \pgfmathsetmacro{\cxD}{12+3/2}
    \fill[KC3] (\cxD-\hrfix,\cxD-\hrfix) rectangle (\cxD+\hrfix,\cxD+\hrfix);
    \foreach \c/\r in {
      15/15,16/16,17/17,18/18,19/19,
      16/15,15/16, 17/16,16/17, 18/17,17/18, 19/15,15/19}{
      \fill[KC4] (\c,\r) rectangle (\c+1,\r+1);}
    \foreach \csi/\cszi/\csj/\cszj/\ca/\cb in {
      0/4/4/4/KC0/KC1,   4/4/0/4/KC1/KC0,
      0/4/8/4/KC0/KC2,   8/4/0/4/KC2/KC0,
      4/4/12/3/KC1/KC3,  12/3/4/4/KC3/KC1,
      8/4/15/5/KC2/KC4,  15/5/8/4/KC4/KC2,
      12/3/15/5/KC3/KC4, 15/5/12/3/KC4/KC3%
    }{
      \pgfmathsetmacro{\cx}{\csi+\cszi/2}
      \pgfmathsetmacro{\cy}{\csj+\cszj/2}
      \fill[\ca] (\cx-\hrfix,\cy-\hrfix) rectangle (\cx,         \cy+\hrfix);
      \fill[\cb] (\cx,        \cy-\hrfix) rectangle (\cx+\hrfix,  \cy+\hrfix);
    }
    \clusterbordersV
    \node[above=4pt, font=\small\bfseries] at (10,0) {CoRe-GNN (batched)};
    \node[below=5pt, font=\small]          at (10,\Ns)
      {$S^{\mathrm{intra}}_{\mathrm{batch}} + S_c^{\mathrm{inter}}$};
  \end{scope}

\end{tikzpicture}%
}
\caption{Propagation matrices for $N=20$ nodes across 5 clusters.
  CoRe-GNN (batched): active clusters (red, violet) use full intra-cluster propagation;
  inactive clusters use a coarsened proxy (small square). The representation of the graph and its coarsened counterpart 
  can be found in Figure~\ref{fig:graph_bilevel_5cluster}.
  }

\label{fig:prop_matrices_v2}
\end{figure}

%% file: tables/planetoid_heterophilic_new.tex

\begin{table}[!h]
\centering
\caption{Test accuracy ($\%$, mean$\pm$std). 
$\mathbf{Bold}$: best. \underline{Underline}: second best.}
\scalebox{0.78}{
\begin{tabular}{lcccccc}
\toprule
 & \multicolumn{3}{c}{Planetoid} & \multicolumn{3}{c}{Heterophilic} \\
\cmidrule(lr){2-4} \cmidrule(lr){5-7}
Model & Cora & CiteSeer & PubMed & chameleon & squirrel & Amazon-ratings \\
\midrule
GCN & $81.58 \pm 0.6$ & $72.40 \pm 0.71$ & $79.11 \pm 0.08$ & $\mathbf{62.68 \pm 2.05}$ & $46.56 \pm 1.87$ & $\mathbf{47.61 \pm 0.58}$ \\
\arrayrulecolor{gray!40}\midrule\arrayrulecolor{black}
Coarsen-GNN & $\mathbf{82.03 \pm 0.3}$ & $\mathbf{73.86 \pm 0.68}$ & $78.94 \pm 0.2$ & $58.03 \pm 2.23$ & $41.31 \pm 1.04$ & $46.99 \pm 0.36$ \\
Cluster-GCN & $76.13 \pm 1.11$ & $71.31 \pm 0.67$ & $78.19 \pm 0.3$ & $58.57 \pm 1.70$ & $42.4 \pm 1.29$ & $47.02 \pm 0.64$ \\
CoRE-GNN & \underline{$81.67 \pm 0.64$} & \underline{$73.62 \pm 0.49$} & $\mathbf{79.27 \pm 0.35}$ & $60.68 \pm 1.99$ & \underline{$46.80 \pm 1.13$} & \underline{$47.45 \pm 0.61$} \\
CoRE-GNN (b) & $80.91 \pm 0.25$ & $73.06 \pm 0.44$ & \underline{$79.26 \pm 0.31$} & \underline{$61.43 \pm 0.81$} & $\mathbf{50.49 \pm 0.93}$ & $45.48 \pm 0.23$ \\
\bottomrule
\end{tabular}
}
\label{tab:results_small}
\end{table}

%% file: tables/large_city_new.tex
\begin{table}[!h]
\centering
\caption{Test accuracy ($\%$, mean$\pm$std).  
$\mathbf{Bold}$: best. \underline{Underline}: second best.}
\scalebox{0.78}{
\begin{tabular}{lccccc}
\toprule
 & \multicolumn{2}{c}{City graphs} & \multicolumn{3}{c}{Large-scale} \\
\cmidrule(lr){2-3} \cmidrule(lr){4-6}
Model & paris & london & ogbn-arxiv & Reddit & ogbn-products \\
\midrule
GCN & $32.9 \pm 1.98$ & $32.41 \pm 1.19$ & $\mathbf{71.86 \pm 0.26}$ & OOM & OOM \\
\arrayrulecolor{gray!40}\midrule\arrayrulecolor{black}
Coarsen-GNN & $35.97 \pm 0.43$ & $35.76 \pm 0.43$ & $61.36 \pm 0.12$ & \underline{$89.87 \pm 0.03$} & \underline{$70.75 \pm 0.19$} \\
Cluster-GCN & $27.43 \pm 1.6$ & $32.28 \pm 0.61 $ & \underline{$70.22 \pm 0.31$} & $89.09 \pm 0.34$ & $67.13 \pm 0.51$ \\
CoRE-GNN & \underline{$42.83 \pm 1.25$} & \underline{$41.1 \pm 0.81$} & $70.03 \pm 0.29$ & OOM & OOM \\
CoRE-GNN (b) & $\mathbf{44.35 \pm 0.45}$ & $\mathbf{44.4 \pm 0.57}$ & $69.56 \pm 0.27$ & $\mathbf{93.41 \pm 0.09}$ & $\mathbf{74.45 \pm 0.27}$ \\
\bottomrule
\end{tabular}
}
\label{tab:results_large}
\end{table}

%% file: tables/complexity.tex
\begin{table}[ht]
\centering
\caption{Time and memory complexity for different GNN training methods. $d$ denotes the hidden dimension, $K$ the number of layers, $E = \mathrm{nnz}(\matrixpropag)$ the number of edges, $N$ the number of nodes, $n = (1-r)N$ the number of super-nodes at coarsening ratio $r$, $E_c = \mathrm{nnz}(\sextrac)$ the edges in the coarsened graph, $E_{\mathrm{intra}} = \mathrm{nnz}(\sintra)$ the intra-cluster edges, $b$ the batch size in nodes and $n_{\mathrm{batch}}$ the number of batches per epoch.
}
\small
\begin{tabularx}{\textwidth}{|c|Y|Y|}
\hline
\textbf{Method} & \textbf{Time Complexity} & \textbf{Memory Complexity} \\
\hline
Classical GCN~\cite{kipf2016GCN} & $\mathcal{O}(K(Ed + Nd^2))$ & $\mathcal{O}(KNd + Kd^2)$ \\
\hline
Coarsen-GNN~\cite{joly2024graph, joly2025taxonomy} & $\mathcal{O}(K(E_cd + nd^2) + 2Nd)$ & $\mathcal{O}(Knd + Nd + Kd^2)$ \\
\hline
Cluster-GCN~\cite{chiang2019cluster} & $\mathcal{O}(K(E_{\mathrm{intra}}d + Nd^2))$ & $\mathcal{O}(Kbd + Kd^2)$ \\
\hline

CoRe-GNN & $\mathcal{O}(K((E_c + E_{\mathrm{intra}})d + (n + N)d^2 + 2Nd))$ & $\mathcal{O}(K(n + N)d + Kd^2)$ \\
\hline
CoRe-GNN (Batched) & $\mathcal{O}(K(n_{\mathrm{batch}}E_cd + E_{\mathrm{intra}}d + (n_{\mathrm{batch}}n + N)d^2))$ & $\mathcal{O}(K(n + b)d + Kd^2)$ \\
\hline
\end{tabularx}
\label{tab:complexity}
\end{table}

%% file: tables/dataset_stats.tex
\begin{table}[ht]
\centering
\caption{Dataset statistics.}
\label{tab:dataset_stats}
\begin{tabular}{lccccc}
\toprule
\textbf{Dataset} & \textbf{\# Nodes} & \textbf{\# Edges} & \textbf{\# Features} & \textbf{\# Classes} & \textbf{Type} \\
\midrule
Cora (PCC) & 2,485 & 10,138 & 1,433 & 7 & Homophilic \\
CiteSeer (PCC) & 2,120 & 7,358 & 3,703 & 6 & Homophilic \\
Pubmed & 19,717 & 88,648 & 500 & 3 & Homophilic \\
\midrule
Chameleon & 2,277 & 36,101 & 2,325 & 5 & Heterophilic \\
Squirrel & 5,201 & 217,073 & 2,089 & 5 & Heterophilic \\
Amazon Ratings & 24,492 & 93,050 & 300 & 5 & Heterophilic \\
\midrule
Reddit & 232,965 & 114,615,892 & 602 & 41 & Large scale \\
ogbn-arxiv & 169,343 & 1,166,243 & 128 & 40 & Large scale \\
ogbn-products & 2,449,029 & 61,859,140 & 100 & 47 & Large scale \\
\midrule
Paris & 114,127 & 182,511 & 37 & 10 & Long range \\
London & 568,795 & 756,502 & 37 & 10 & Long range \\
\bottomrule
\end{tabular}
\end{table}

%% file: tables/best_hp_for_table.tex

\begin{table}[!h]
\centering
\caption{Best Hyperparameters CoRE-GNN}
\scalebox{1}{
\begin{tabular}{llccccc}
\toprule
Dataset & Coarsened graph & lr & wd & hidden & layers & LN \\
\midrule
Cora & loukas r 0.5 & 0.01 & 0.001 & 256 & 2 & F \\
CiteSeer & loukas r 0.9 & 0.01 & 0 & 256 & 2 & F \\
PubMed & metis nclusters 500 & 0.01 & 0.001 & 128 & 2 & F \\
chameleon & loukas r 0.9 & 0.01 & 0.001 & 512 & 2 & F \\
squirrel & loukas r 0.3 & 0.01 & 0.0005 & 512 & 2 & F \\
Amazon-ratings & metis nclusters 50 & 0.001 & 0.001 & 512 & 2 & T \\
paris & loukas r 0.95 & 0.01 & 0 & 512 & 5 & T \\
london & graclus r 0.9 & 0.01 & 0 & 512 & 5 & T \\
ogbn-arxiv & graclus r 0.999 & 0.001 & 0.0001 & 512 & 3 & T \\
\bottomrule
\end{tabular}
}
\label{tab:hp_coregnn}
\end{table}


\begin{table}[!h]
\centering
\caption{Best Hyperparameters CoRE-GNN (b)}
\scalebox{1}{
\begin{tabular}{llcccccc}
\toprule
Dataset & Coarsened graph & lr & wd & hidden & layers & LN & batch \\
\midrule
Cora & loukas r 0.5 & 0.001 & 0.0005 & 512 & 2 & F & 5000 \\
CiteSeer & loukas r 0.9 & 0.001 & 0 & 512 & 3 & F & 5000 \\
PubMed & metis nclusters 500 & 0.001 & 0.0005 & 512 & 2 & F & 5000 \\
chameleon & loukas r 0.9 & 0.0001 & 0 & 512 & 2 & T & 5000 \\
squirrel & loukas r 0.3 & 0.0001 & 0.0005 & 256 & 2 & T & 5000 \\
Amazon-ratings & loukas r 0.5 & 0.001 & 0 & 256 & 2 & F & 5000 \\
paris & loukas r 0.9 & 0.001 & 0.0005 & 256 & 5 & T & 5000 \\
london & graclus r 0.9 & 0.001 & 0 & 256 & 5 & T & 5000 \\
ogbn-arxiv & metis nclusters 100 & 0.001 & 0.0005 & 512 & 3 & T & 5000 \\
Reddit & loukas r 0.999 & 0.001 & 0 & 256 & 4 & T & 5000 \\
ogbn-products & metis nclusters 1500 & 0.0001 & 0 & 256 & 3 & T & 5100 \\
\bottomrule
\end{tabular}
}
\label{tab:hp_coregnn_b}
\end{table}


\begin{table}[!h]
\centering
\caption{Best Hyperparameters Coarsen-GNN}
\scalebox{1}{

\begin{tabular}{llccccc}
\toprule
Dataset & Coarsened graph & lr & wd & hidden & layers & LN \\
\midrule
Cora & loukas r 0.3 & 0.01 & 0.001 & 512 & 2 & F \\
CiteSeer & loukas r 0.7 & 0.01 & 0.001 & 512 & 2 & F \\
PubMed & graclus r 0.5 & 0.001 & 0.001 & 512 & 2 & F \\
chameleon & loukas r 0.3 & 0.01 & 0.001 & 256 & 2 & F \\
squirrel & loukas r 0.3 & 0.01 & 0.001 & 256 & 2 & F \\
Amazon-ratings & loukas r 0.3 & 0.001 & 0.001 & 512 & 2 & T \\
paris & loukas r 0.9 & 0.01 & 0 & 512 & 4 & T \\
london & graclus r 0.95 & 0.01 & 0 & 512 & 6 & T \\
ogbn-arxiv & graclus r 0.9 & 0.01 & 0 & 512 & 2 & T \\
Reddit & loukas r 0.9 & 0.01 & 0 & 256 & 2 & F \\
ogbn-products & metis nclusters 1500 & 0.01 & 0 & 128 & 3 & F \\
\bottomrule
\end{tabular}
}
\label{tab:hp_coarsengnn}
\end{table}


\begin{table}[!h]
\centering
\caption{Best Hyperparameters Cluster-GCN}
\scalebox{1}{

\begin{tabular}{llcccccc}
\toprule
Dataset & Coarsened graph & lr & wd & hidden & layers & LN & batch \\
\midrule
Cora & loukas r 0.9 & 0.001 & 0.0005 & 512 & 3 & T & 5000 \\
CiteSeer & graclus r 0.95 & 0.0001 & 0.0005 & 512 & 3 & F & 5000 \\
PubMed & metis nclusters 50 & 0.001 & 0.0005 & 256 & 2 & F & 5000 \\
chameleon & graclus r 0.95 & 0.0001 & 0.0005 & 512 & 2 & T & 5000 \\
squirrel & loukas r 0.95 & 0.0001 & 0.0005 & 256 & 2 & T & 5000 \\
Amazon-ratings & metis nclusters 100 & 0.0001 & 0.0005 & 512 & 2 & T & 5000 \\
paris & loukas r 0.999 & 0.001 & 0 & 512 & 6 & F & 5000 \\
london & graclus r 0.99 & 0.001 & 0 & 256 & 6& T & 5000 \\
ogbn-arxiv & graclus r 0.999 & 0.001 & 0 & 512 & 3 & T & 5000 \\
Reddit & loukas r 0.95 & 0.001 & 0 & 512 & 2 & F & 5000 \\
ogbn-products & metis nclusters 1500 & 0.001 & 0 & 256 & 3 & T & 5000 \\
\bottomrule
\end{tabular}
}
\label{tab:hp_clustergcn}
\end{table}


\begin{table}[!h]
\centering
\caption{Best Hyperparameters GCN}
\scalebox{1}{
\begin{tabular}{lccccc}
\toprule
Dataset & lr & wd & hidden & layers & LN \\
\midrule
Cora & 0.001 & 0.0005 & 512 & 4 & F \\
CiteSeer & 0.001 & 0.0005 & 512 & 6 & F \\
PubMed & 0.001 & 0.0005 & 512 & 2 & F \\
chameleon & 0.01 & 0.0005 & 256 & 2 & F \\
squirrel & 0.01 & 0.0005 & 64 & 2 & F \\
Amazon-ratings & 0.01 & 0.0001 & 512 & 2 & F \\
paris & 0.01 & 0 & 256 & 6 & T \\
london & 0.01 & 0 & 256 & 6 & T \\
ogbn-arxiv & 0.001 & 0.0001 & 512 & 3 & T \\
\bottomrule
\end{tabular}
}
\label{tab:hp_gcn}
\end{table}

%% file: tables/hyperparameter_grid.tex
\begin{table*}[t]
\small
\caption{Hyperparameter search grids on 3 independent runs}
\label{tab:hyperparameter_grid}

\begin{minipage}[t]{0.46\textwidth}
\centering
\textbf{(a) CoRE-GNN}\\[4pt]
\begin{tabular}{lc}
\toprule
Hyperparameter & Values \\
\midrule
\multicolumn{2}{l}{\textit{Coarsening}} \\
Algorithm           & Metis, Graclus, Loukas \\
Reduction ratio $r$ & \{0.3, 0.5, 0.7, 0.9, 0.95\} \\
Metis $n_c$         & \{50, 100, 200, 500, 1000\}$^{\ddagger}$ \\
\midrule
\multicolumn{2}{l}{\textit{Training}} \\
Learning rate   & \{$10^{-2}$,\,$10^{-3}$,\,$10^{-4}$\} \\
Weight decay    & \{$0$,\,$10^{-4}$,\,$5{\times}10^{-4}$,\,$10^{-3}$\} \\
Dropout         & 0.5 \\
No.\ layers     & \{2, 3, 4\}$^{*}$ \\
Hidden dim.     & \{64, 128, 256, 512\} \\
Epochs          & 800 \\
Layer norm      & \{yes, no\}$^{\#}$ \\
\bottomrule
\end{tabular}
\end{minipage}
\hfill
\begin{minipage}[t]{0.46\textwidth}
\centering
\textbf{(b) Cluster-GCN}\\[4pt]
\begin{tabular}{lc}
\toprule
Hyperparameter & Values \\
\midrule
\multicolumn{2}{l}{\textit{Coarsening}} \\
Algorithm           & Metis, Graclus, Loukas \\
Reduction ratio $r$ & \{0.3, 0.95, 0.99, 0.999\} \\
Metis $n_c$         & \{50, 100, 1500\}$^{\ddagger}$ \\
\midrule
\multicolumn{2}{l}{\textit{Training}} \\
Learning rate   & \{$10^{-3}$,\,$10^{-4}$\} \\
Weight decay    & \{$0$,\,$5{\times}10^{-4}$\} \\
Dropout         & 0.5 \\
No.\ layers     & \{2, 3, 4, 6\} \\
Hidden dim.     & \{256, 512\} \\
Epochs          & 800 \\
Layer norm      & \{yes, no\} \\
Node batch size & 5\,000 (large only) \\
\bottomrule
\end{tabular}
\end{minipage}

\vspace{1.4em}

\centering
\textbf{(c) CoRE-GNN batched}\\[4pt]
\resizebox{\textwidth}{!}{%
\begin{tabular}{lcccc}
\toprule
Hyperparameter
  & Small / Medium
  & ogbn-arxiv / Reddit
  & ogbn-products
  & City \\
\midrule
\multicolumn{5}{l}{\textit{Coarsening}} \\
Algorithm
  & Metis,  Graclus, Loukas,
  & Metis,  Graclus, Loukas
  & Metis, Graclus
  & Metis, Graclus, Loukas \\
Reduction ratio $r$
  & \{0.3, 0.5, 0.7, 0.9\}
  & \{0.9, 0.95, 0.99, 0.999\}
  & 0.999
  & \{0.9, 0.95, 0.99, 0.999\} \\
Metis $n_c$
  & \{50,500\}
  & \{100, 500, 1500\}
  & \{500, 1500\}
  & \{100, 500, 1000, 1500\} \\
\midrule
\multicolumn{5}{l}{\textit{Training}} \\
Learning rate
  & \{$10^{-3}$, $10^{-4}$\}
  & \{$10^{-3}$, $10^{-4}$\}
  & $10^{-4}$
  & $10^{-3}$ \\
Weight decay
  & \{$0$, $5{\times}10^{-4}$\}
  & $5{\times}10^{-4}$
  & $0$
  & $5{\times}10^{-4}$ \\
Dropout         & 0.5  & 0.5  & 0.5  & 0.5  \\
No.\ layers     & \{2, 3\} & \{3, 4\} & 3 & \{2,\ldots,24\}$^{*}$ \\
Hidden dim.     & \{256, 512\} & \{256, 512\} & 256 & 256 \\
Epochs          & 800  & 800  & 800  & 800  \\
Layer norm      & \{yes, no\} & yes & yes & yes \\
Node batch size & 500 & 5\,000 & 5\,100 & 5\,000 \\
\bottomrule
\end{tabular}%
}

\vspace{1.4em}

\begin{minipage}[t]{0.46\textwidth}
\centering
\textbf{(d) Coarsen-GNN}\\[4pt]
\begin{tabular}{lc}
\toprule
Hyperparameter & Values \\
\midrule
\multicolumn{2}{l}{\textit{Coarsening}} \\
Algorithm           & Metis, Graclus, Loukas \\
Reduction ratio $r$ & \{0.3, 0.5, 0.7, 0.9, 0.95\} \\
Metis $n_c$         &  \{50, 100, 200, 500, 1000,1500\}$^{\ddagger}$ \\
\midrule
\multicolumn{2}{l}{\textit{Training}} \\
Learning rate   & \{$10^{-2}$,\,$10^{-3}$\} \\
Weight decay    & \{$0$,\,$5{\times}10^{-4}$,\,$10^{-3}$\} \\
Dropout         & 0.5 \\
No.\ layers     & \{2, 3, 4, 6\} \\
Hidden dim.     & \{256, 512\} \\
Epochs          & 800 \\
Layer norm      & \{yes, no\} \\
\bottomrule
\end{tabular}
\end{minipage}
\hfill
\begin{minipage}[t]{0.46\textwidth}
\centering
\textbf{(e) GCN}\\[4pt]
\begin{tabular}{lc}
\toprule
Hyperparameter & Values \\
\midrule
Learning rate   & \{$10^{-2}$,\,$10^{-3}$\} \\
Weight decay    & \{$0$,\,$10^{-4}$,\,$5{\times}10^{-4}$\} \\
Dropout         & 0.5 \\
No.\ layers     & \{2, 3, 4, 5, 6\}$^{*}$ \\
Hidden dim.     & \{64, 128, 256, 512\} \\
Epochs          & 800 \\
Layer norm      & \{yes, no\} \\
\bottomrule
\end{tabular}
\end{minipage}

\vspace{1.4em}

\vspace{1ex}
{
$^{\ddagger}$~$n_c$ = number of Metis communities, Graclus and Loukas use reduction ratio $r$ directly.\\
$^{\#}$~LN choice is dataset-dependent: LN$=$False for Cora/CiteSeer/PubMed/chameleon/squirrel; LN$=$True for Amazon-ratings, and all large graphs.\\
$^{*}$~Cities: CoRE-GNN and GCN tested up to 24 layers.\\
}
\end{table*}

%% file: tikz_figure/graph_bilevel_5cluster.tikz
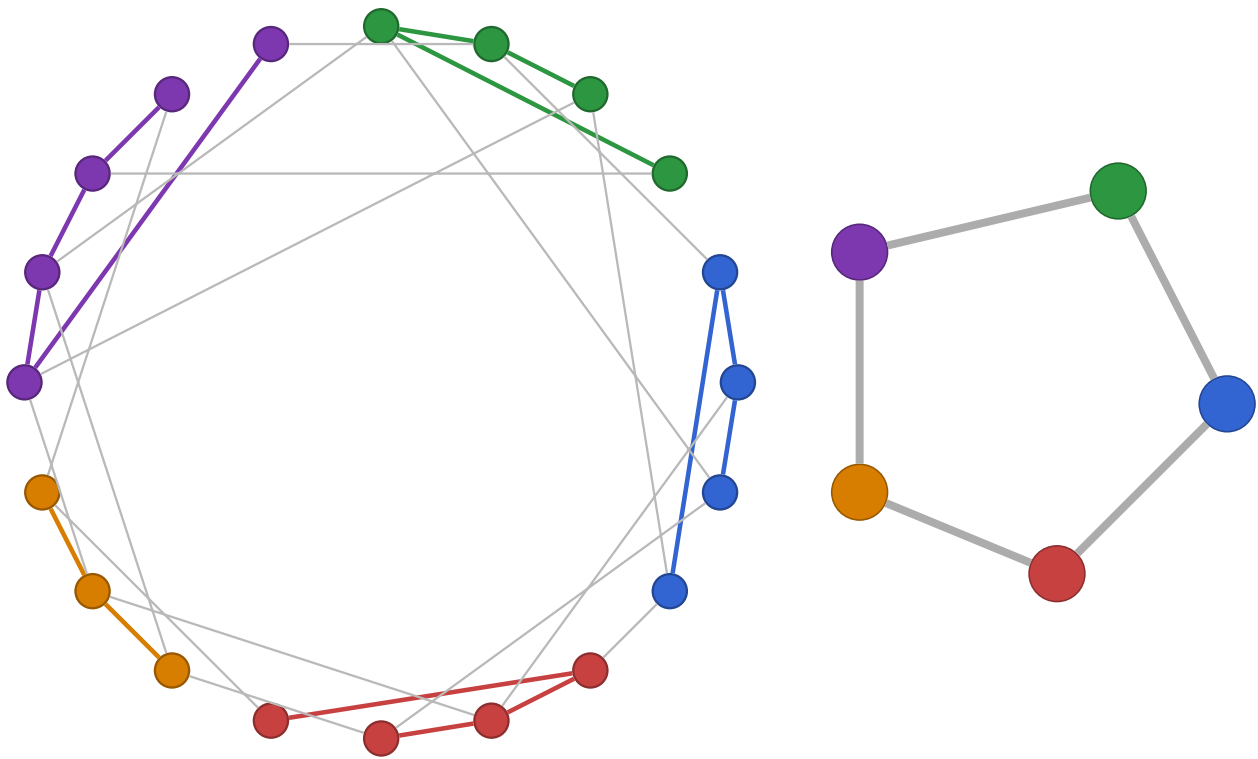
\begin{figure}[!h]
\centering
\begin{subfigure}[c]{0.6\linewidth}
\centering
\resizebox{\linewidth}{!}{%
\begin{tikzpicture}[
  gnode/.style={circle, fill=#1, draw=#1!70!black, minimum size=6pt, inner sep=0pt},
  gnode/.default=black,
]
\definecolor{KC0}{RGB}{50,100,210}
\definecolor{KC1}{RGB}{200,65,65}
\definecolor{KC2}{RGB}{45,150,65}
\definecolor{KC3}{RGB}{215,125,0}
\definecolor{KC4}{RGB}{125,55,175}

\node[gnode=KC2] (N8)  at (2.50, 4.20) {};
\node[gnode=KC2] (N9)  at (3.18, 4.09) {};
\node[gnode=KC2] (N10) at (3.79, 3.78) {};
\node[gnode=KC2] (N11) at (4.28, 3.29) {};
\node[gnode=KC0] (N0)  at (4.59, 2.68) {};
\node[gnode=KC0] (N1)  at (4.70, 2.00) {};
\node[gnode=KC0] (N2)  at (4.59, 1.32) {};
\node[gnode=KC0] (N3)  at (4.28, 0.71) {};
\node[gnode=KC1] (N4)  at (3.79, 0.22) {};
\node[gnode=KC1] (N5)  at (3.18,-0.09) {};
\node[gnode=KC1] (N6)  at (2.50,-0.20) {};
\node[gnode=KC1] (N7)  at (1.82,-0.09) {};
\node[gnode=KC3] (N12) at (1.21, 0.22) {};
\node[gnode=KC3] (N13) at (0.72, 0.71) {};
\node[gnode=KC3] (N14) at (0.41, 1.32) {};
\node[gnode=KC4] (N15) at (0.30, 2.00) {};
\node[gnode=KC4] (N16) at (0.41, 2.68) {};
\node[gnode=KC4] (N17) at (0.72, 3.29) {};
\node[gnode=KC4] (N18) at (1.21, 3.78) {};
\node[gnode=KC4] (N19) at (1.82, 4.09) {};

\draw[KC0, thick] (N1)--(N0) (N2)--(N1) (N3)--(N0);
\draw[KC1, thick] (N5)--(N4) (N6)--(N5) (N7)--(N4);
\draw[KC2, thick] (N9)--(N8) (N10)--(N9) (N11)--(N8);
\draw[KC3, thick] (N13)--(N12) (N14)--(N13);
\draw[KC4, thick] (N16)--(N15) (N17)--(N16) (N18)--(N17) (N19)--(N15);

\draw[gray!55, thin]
  (N5)--(N1) (N4)--(N3) (N6)--(N2)
  (N9)--(N0) (N8)--(N2) (N10)--(N3)
  (N12)--(N6) (N14)--(N7) (N13)--(N5)
  (N15)--(N10) (N17)--(N11) (N19)--(N9) (N16)--(N8)
  (N16)--(N12) (N18)--(N14) (N15)--(N13);
\end{tikzpicture}%
}
\end{subfigure}
\hfill
\begin{subfigure}[c]{0.34\linewidth}
\centering
\resizebox{\linewidth}{!}{%
\begin{tikzpicture}[
  snode/.style={circle, fill=#1, draw=#1!70!black, minimum size=14pt, inner sep=0pt},
  snode/.default=black,
]
\definecolor{KC0}{RGB}{50,100,210}
\definecolor{KC1}{RGB}{200,65,65}
\definecolor{KC2}{RGB}{45,150,65}
\definecolor{KC3}{RGB}{215,125,0}
\definecolor{KC4}{RGB}{125,55,175}

\node[snode=KC2] (S2) at (3.32, 3.60) {};
\node[snode=KC0] (S0) at (4.28, 1.72) {};
\node[snode=KC1] (S1) at (2.78, 0.22) {};
\node[snode=KC3] (S3) at (1.04, 0.94) {};
\node[snode=KC4] (S4) at (1.04, 3.06) {};

\draw[gray!65, line width=2pt]
  (S0)--(S1) (S0)--(S2) (S1)--(S3) (S2)--(S4) (S3)--(S4);
\end{tikzpicture}%
}
\end{subfigure}
\caption{Illustration of the coarsening on the same 5-cluster partition as Figure~\ref{fig:prop_matrices_v2}. Left: original graph ($N=20$ nodes). Right: coarsened graph.}
\label{fig:graph_bilevel_5cluster}
\end{figure}